\documentclass[final]{hooml2022} 
 
\usepackage[textsize=tiny]{todonotes}
\usepackage{algorithm,algorithmic}

\newcommand{\ls}{\left(}
\newcommand{\rs}{\right)}
\newcommand{\lb}{\left\lbrace}
\newcommand{\rb}{\right\rbrace}
\newcommand{\lp}{\left[}
\newcommand{\rp}{\right]}
\newcommand{\la}{\left\langle}
\newcommand{\ra}{\right\rangle}
\newcommand{\EE}[2]{{\mathbb E}_{#1}\left[#2\right] } 
\newcommand{\R}{\mathbb{R}}
\newcommand{\B}{\widetilde{B}}
\newcommand{\eqdef}{:=}
\newtheorem{assumption}{Assumption}

\newcommand{\Y}{\widetilde{Y}}
\newcommand{\bY}{\widetilde{Y}}
\newcommand{\A}{\mathcal{A}}

\def\bk{\bar \kappa}
\newcommand{\wg}{\widetilde{g}}
\newcommand\numberthis{\addtocounter{equation}{1}\tag{\theequation}}

\newtheorem{innercustomthm}{Theorem}
\newenvironment{customthm}[1]
  {\renewcommand\theinnercustomthm{#1}\innercustomthm}
{\endinnercustomthm}

\newtheorem{innercustomassum}{Assumption}
\newenvironment{customassum}[1]
  {\renewcommand\theinnercustomassum{#1}\innercustomassum}
{\endinnercustomassum}  

\newtheorem{innercustomcol}{Corollary}
\newenvironment{customcol}[1]
  {\renewcommand\theinnercustomcol{#1}\innercustomcol}
{\endinnercustomcol}

\def\aa#1{{\color{black}#1}} 
\def\br#1{{\color{black}#1}} 

\def\algo{FLECS-CG{D} }

\title{FLECS-CGD: A Federated Learning Second-Order Framework via Compression and Sketching with Compressed Gradient Differences}

\optauthor{%
\Name{Artem Agafonov} \Email{agafonov.ad@phystech.edu}\\
\addr Mohamed bin Zayed University of Artificial Intelligence, Abu Dhabi, UAE\\
\addr Moscow Institute of Physics and Technology, Dolgoprudny, Russia
\AND
\Name{Brahim Erraji} \Email{brahim.erraji@mbzuai.ac.ae}\\
\addr Mohamed bin Zayed University of Artificial Intelligence, Abu Dhabi, UAE
\AND
\Name{Martin Tak\'a\v{c}} \Email{martin.takac@mbzuai.ac.ae}\\
\addr Mohamed bin Zayed University of Artificial Intelligence, Abu Dhabi, UAE
}

\begin{document}

\maketitle
\vspace{-14mm}
\begin{abstract}%
In the recent paper FLECS (Agafonov et al, FLECS: A Federated Learning Second-Order Framework via Compression and Sketching), the second-order framework FLECS was proposed for the Federated Learning problem. This method utilize compression of sketched Hessians to make communication costs low. However, the main bottleneck of FLECS is gradient communication without compression. In this paper, we propose the modification of FLECS with compressed gradient differences, which we call \algo (FLECS with \textbf{C}ompressed \textbf{G}radient \textbf{D}ifferences) and make it applicable for stochastic optimization. Convergence guarantees are provided in strongly convex and nonconvex cases. Experiments show the practical benefit of proposed approach.
\end{abstract}


\section{Introduction}
    In this paper, we focus on the stochastic federated learning problem, where the objective function is the empirical loss of overall $n$ workers:
    \vspace{-2mm}
    \begin{equation}\label{eq:problem}
        \min\limits_{w \in \R^d} \lb F(w):= \frac{1}{n}\sum\limits_{i=1}^n f_i(w)\rb,
    \end{equation}
    \vspace{-0.75mm}
    where 
    $f_i(w) = \EE{\xi \sim \mathcal{D}_i}{f(w, \xi)},$
    with $F$ being a general loss function (parametrized by $w \in \R^d$ and $\xi$) associated with the data stored on the $i$-th machine, $n$ is number of machines. The
    distributions $\mathcal{D}_{1},\ldots,\mathcal{D}_{n}$ may differ on each machine,
    so means the functions $f_1, \ldots, f_n$ can have completely
    different minimizers. In particular, $\EE{\xi \sim \mathcal{D}_i}{\nabla f_i(w^*, \xi)} \neq 0$, where $w^*$ is an optimal solution of \eqref{eq:problem}.
    
    Problems of this nature are ubiquitous and rise naturally whenever multiple computing nodes can be connected \cite{LI2020106854,https://doi.org/10.48550/arxiv.1811.03604,DBLP:journals/corr/McMahanMRA16}. For example, such problems arise in distributed machine learning, robotics, resource allocation, optimal transport, and other applications \cite{kra13,jaggi2014communication,mousavi2019multi,smith2018cocoa, li2014scaling, xia06, rab04, nedic2017fast, dvurechenskii2018decentralize,richtarik2016distributed,marecek2014distributed,ma2017distributed, uribe2018distributed, ram2009distributed}.

    The most popular approaches for the problem \eqref{eq:problem} might be first-order algorithms. Early work in that field includes \cite{Konecny2016a,Konecny2016b,McMahan2017}. To reduce communication burden  methods with gradient compression were proposed  \cite{alistarh2017qsgd,horvath2019natural,mishchenko2019distributed,chen2022distributed}.  Other techniques like the usage of momentum \cite{Mishchenko2019,Li2020accFL}, variance reduction \cite{Horvath2019,nguyen2021inexact,nguyen2017sarah},
    and adaptive learning rates
    \cite{shi2021ai}
    were recently proposed in the literature.
    
    Second-order methods are also proposed for the Federated Learning setup. Generally, these methods can be divided into two groups based on heterogeneous/homogeneous data setting assumptions. Algorithms in the first group \cite{Zhang2018,dvurechensky2021hyperfast,daneshmand2021newton,bullins2021stochastic, agafonov21acc} usually utilize statistical similarity, which means that the local function $f_i$ approximates the global objective $F$ well. Methods  FedNL \cite{safaryan2021fednl} and FLECS \cite{agafonov2022flecs} work in truly heterogeneous setup, which makes them more practical. However, FedNL seems impractical for large-scale problems because of high memory requirements  on devices.    Indeed, in FedNL it is assumed that each device (e.g. mobile phone) should store $d \times d$ Hessian approximation locally, which is impossible for large $d$. FLECS tackles this problem by using a sketching technique and switching memory costs from machines to the server. Therefore, FLECS \aa{does not require storing Hessians locally}.
    In both FedNL and FLECS, compression is applied only to the (sketched) difference of Hessian and its approximation. The goal of this work is to add gradient compression to FLECS. That allows reducing communication complexity. Moreover, compared to \cite{agafonov2022flecs}, we show that FLECS and CG-FLECS work in the general stochastic distributed optimization problem.
    \vspace{-2mm}
    \paragraph{Contribution}
    We briefly describe our contributions below. First of all, we make FLECS appliable to the the stochastic federated learning. Secondly, we propose \algo --  the modification of FLECS with gradient compression. This improves communication complexity from $O(cmd + {32}d + \aa{32m^2})$ (float32) to $O(cmd + cd + \aa{32m^2})$, where $m$ is a user-defined memory size and $c$ is a number of bits per one value after compression (typically $c \ll 32$). Thirdly, we provide theoretical convergence guarantee in non-convex and strongly convex cases. Finally, our numerical experiments show practical benefit of the proposed approach.
    \vspace{-2mm}
    \paragraph{Organization}
    The rest of the paper is organized as follows. In Section \ref{sec:prel}, we introduce main notations and definitions. Then, In Section \ref{sec:algo} we present our framework FLECS with compressed gradients for stochastic distributed optimization problem  \eqref{eq:problem}. Section \ref{sec:theory} is dedicated to the convergence analysis of the proposed method (all proofs can be found in appendix). Finally, numerical experiments are provided in Section \ref{sec:experiments}.
\vspace{-3mm}    
\section{Preliminaries}\label{sec:prel}
\begin{definition}
    A differentiable function $F: \R^d \to \R^d$ is called $\mu$-strongly convex for $\mu > 0$, if for all $x, ~ y \in \R^d:~~~$
    $
        F(x) + \langle \nabla F(x), y - x\rangle + \frac{\mu}{2}\|x-y\|^2 \leq F(y).
    $
\end{definition}

\begin{definition}
    A differentiable function $F: \R^d \to \R^d$ is called $L$-smooth for $L > 0$, if for all $x, ~ y  \in \R^d:~~~$
    $
        F(y) \leq  F(x) + \langle \nabla F(x), y - x\rangle + \frac{L}{2}\|x-y\|^2. 
    $
\end{definition}

\begin{definition}
    By $\mathcal{U}(\omega)$ ($\omega > 0$) we define the class of unbiased compression operators $Q:~\R^d \to \R^d$ satisfying 
    \begin{equation}\label{eq:quantization_def}
        \EE{Q}{Q(x)} = x, ~  ~ \EE{Q}{\|Q(x)\|^2} \leq (\omega + 1)\|x\|^2,
    \end{equation}
    for all $x \in \R^d$.
\end{definition}


\vspace{-3mm}
\section{\algo: FLECS with Gradient Compression}\label{sec:algo}

 In this section, we describe the main steps of the proposed method.\algo is a modification of FLECS with gradient compression. \algo is listed as Algorithm \ref{alg:main}. Detailed information about FLECS can be found in the original article \cite{agafonov2022flecs}.
 
 The algorithm is initialized with user-defined memory size $m \ll d$, $n$ vectors $h^i_0 \in \R^d$ and $n$ matrices $B^i_0 \in \R^{d \times d}$. Each $h^i_0$ ($B_0^i$) represents an approximation to the local gradient (Hessian) for the $i$-th worker.
 
 In the beginning of iteration $k$ each worker receives $B_k^iS_k, ~ w_k$ from the server. Then, $i$-th worker samples $S_k \in \R^{d\times m}$, $g_k^i$ such that $\EE{}{g_k^i \vert w_k} = \nabla f_i(w_k)$, $H_k^iS_k$ such that $\EE{}{H_k^i S_k \vert w_k, S_k} = \nabla^2 f_k^i(w_k)S_k$.  Note that $S_k$ is the same for all machines and the server; we guarantee it by setting the random seed to be equal to the iteration number $k$.
 
 Next, to utilize error-feedback technique each worker calculates 
 \begin{equation}
     c_k^i = {Q}(g_k^i - h_k^i) \in \R^d,~ M_k^i=S_k^TY_k^i \in \R^{m \times m},~ C_k^i = \mathcal{C}(Y_k^i - B_k^i S_k) \in \R^{d \times m}.
 \end{equation}
Then the $i$-th worker sends compressed differences $c_k^i, ~ C_k^i$ and $M_k^i$ to the server. 

The server receives  $c_k^i, ~ C_k^i, ~M_k^i$ from all $i=1, \ldots, n$ workers. Firstly, $\wg_k^i = c_k^i + h_k^i$ and $\Y_k^i = C_k^i + B_k^iS_k$ are computed. Then,  the server computes new Hessian approximation $B_{k+1}^i, ~ i=1, \ldots, n$ via Truncated L-SR1 update (Algorithm \ref{alg:hess_lsr1}) or Direct update (Algorithm \ref{alg:hess_direct}).  

At the very end of $k$-th iteration the server forms
\begin{align*}
      \wg_k &= \tfrac{1}{n} \textstyle{\sum} _{i=1}^n \wg_k^i,  &
      \widetilde{Y}_k &\eqdef \tfrac{1}{n}\textstyle{\sum} \Y_k^i = \tfrac{1}{n}\textstyle{\sum} _{i=1}^n (C_k^i + H_k^i S_k), \\
      M_k &\eqdef \tfrac{1}{n} \textstyle{\sum}_{i=1}^n M_k^i = S_k \nabla^2 F(w_k)S_k^T, 
      &
      B_{k+1} &\eqdef \tfrac{1}{n} \textstyle{\sum}_{i=1}^n B_{k+1}^i.
\end{align*}
Finally, the main node calculates new iterate $w_{k+1}$ via update rule $
            w_{k+1} = w_k + \alpha_k p_k
$
where $\alpha_k > 0$ is the step-size. Search direction $p_k$ can be computed via truncated inverse Hessian approximation step (Algorithm \ref{alg:approx_trunc}) or via FedSONIA (Algorithm \ref{alg:approx_fedsonia}) step.
\vspace{-2mm}
\paragraph{Communication complexity}
Omitting both gradient and matrix compressions communication complexities per node of FLECS and FLECS-CGD are the same. Both algorithms need to send $d$ dimensional vector, one $m \times m$ matrix and one $d \times m$ matrix. However, when using compression, the situation is different.
 Assuming that float data type is used, FLECS-CG reduces communication complexity  of FLECS $O(cmd + {32}d + \aa{32m^2})$ to $O(cmd + cd + 32\aa{m^2})$, where $c$ is number of bits per digit. It is important for the practical case of small memory sizes $m$. Indeed, if we set $m=1$, then FLECS-CG communication complexity is $O(cd)$ which is much smaller than FLECS's $O(32d)$.
\vspace{-2mm}
\paragraph{Step complexity \cite{agafonov2022flecs}} The the worker step's complexity consists of $m$ Hessian-vector products and matrix multiplication $(O(md^2))$. The total complexity of either Hessian approximation (Algorithms \ref{alg:hess_lsr1}, \ref{alg:hess_direct}) update is $O(nmd^2)$ . So the server step's complexity depends on options for the search direction: $O(d^3 + nmd^2)$ for Truncated Inverse Hessian approximation (Algorithm \ref{alg:approx_trunc}) and $O(nmd^2)$ for FedSONIA (Algorithm \ref{alg:approx_fedsonia}) .

    \begin{algorithm}[t]
           \caption{\algo}\label{alg:main}
            \begin{algorithmic}[1]
            \REQUIRE $w_0$~-- starting point,
            $m$~-- memory size, $B_0^i$~-- initial Hessian approximations for each worker $i=1\ldots n$ on the server, $0 < \omega < \Omega$ -- truncation constants.
            \FOR{$k=0,1,\ldots$}
            \STATE \textbf{On $i$-th machine:}
            \STATE collect $B_k^iS_k, ~ w_k$ from the server;  
            \STATE sample $S_k \in \R^{d\times m}$, \aa{$g_k^i$ such that $\EE{}{g_k^i \vert w_k} = \nabla f_i(w_k)$, $ H_k^i S_k$ such that $\EE{}{H_k^i S_k \vert w_k, S_k} = \nabla^2  f_i(w_k)S_k$}; 
            \STATE let $Y_k^i:= H_k^iS_k$, and compute $M_k^i:= S_k^T Y_k^i$;
            \STATE send $\aa{c_k^i = {Q}(g_k^i - h_k^i)},~ M_k^i,~ C_k^i = \mathcal{C}(Y_k^i - B_k^i S_k)$ to the server.
            \STATE \aa{select stepsize $\gamma_k$ and update $h^i_{k+1} := h^i_k + \gamma_k {c}_k^i$}
            \STATE \textbf{On the server:}
            \STATE sample $S_k$;
            \STATE collect $C_k^i,~ M_k^i,~ \aa{c_k^i} ~ i=1\ldots n$ from workers;
            \STATE compute $\aa{\widetilde{g}_k^i = c_k^i + h_k^i}, ~ \bY_k^i = C_k^i + B_k^iS_k$;
            \STATE compute $B_{k+1}^i$ via Algorithm \ref{alg:hess_lsr1} or select learning rate $\beta
            _k$ and compute  $B_{k+1}^i$ via Algorithm \ref{alg:hess_direct};
            \STATE form $\aa{\wg_k}, \bY_k, M_k, B_{k+1}$ as average over workers of $ \aa{\wg_k^i}, \bY_k^i, M^i_k, B^i_{k+1}, ~ i=1, \ldots, n$;
            \STATE compute search direction $p_k$ via Algorithm \ref{alg:approx_trunc} or \ref{alg:approx_fedsonia};
            \STATE select stepsize $\alpha_k$ and set $w_{k+1} = w_k + \alpha_k p_k$;
            \STATE sample $S_{k+1} \in \R^{d\times m}$;
            \STATE send $w_k, B_{k+1}S_{k+1}$ to all workers.
            \ENDFOR
        \end{algorithmic}
    \end{algorithm}
\vspace{-5mm}
\section{\algo: Convergence Analysis}\label{sec:theory}
    In this section, we present the convergence theory for \algo. All proofs can be found in Appendix \ref{app:experiments}. Let $w_0$ be an initial point and $w^*$ be a solution: 
    $w^{\star} = \arg \min \limits_{w \in \R^d} F(w), ~ \text{and}~F^{\star} = F(w^{\star})$.\\ 
    Before proving the convergence of \algo for different classes of functions, we will cite a few key assumptions and lemmas.
    \vspace{-2mm}
    \begin{assumption}\label{assum:diff}
            The function $F$ is twice continuously differentiable.
    \end{assumption}
    \vspace{-2mm}
    First, we focus on strongly convex case and present assumptions for this setting.
    \vspace{-2mm}
    \begin{assumption}\label{assum:strong_conv}
        Each function $f_i(w)$ is $\mu$-strongly convex and $L$-smooth
        $
            \mu I \preceq \nabla^2 f_i(w) \preceq L I.
        $
    \end{assumption}
    
    \vspace{-5mm}
    \begin{assumption}\label{assum:variance}
        Each $g_k^i$ in Algorithm \ref{alg:main} has bounded variance 
            $
            \EE{}{\|g_k^i - \nabla f_i(w_k)\|} \leq \sigma_i^2,$ \\
            $ \forall k \geq 0, ~ i = 1, \ldots, n~~$
        for constants $\sigma_i < \infty$, $\sigma^2 \eqdef \frac{1}{n} \sum \limits_{i=1}^n \sigma_i^2$.
    \end{assumption}
    \vspace{-4mm}
    The following theorem establishes global linear convergence of \algo under strong convexity.
    \vspace{-3mm}
    \begin{theorem}\label{thm:str_conv}
        Suppose that Assumption \ref{assum:diff}, \ref{assum:strong_conv}, \ref{assum:variance} holds. Let $Q \in \mathcal{U}(\omega)$.
         Let $\{w_k\}$ be the iterates generated by Algorithm~\ref{alg:main}, where $0 <  \alpha_k = \alpha  \leq \frac{5\mu \mu_1}{2L^2\mu_2^2\left(1+\frac{\omega}{n}\right)}$ and $0 < \gamma_k = \gamma \leq \frac{1}{\omega + 1}$. Define the Lyapunov function
         $\Psi_{k+1} = (F(w_{k+1}) - F(w_*)) + \frac{cL\mu_2^2\alpha^2}{n}\sum \limits_{i=1}^n \EE{Q}{\|h_{k+1}^i - h_*^i\|^2}$
         for \\
         $0 < c = \min \lb \frac{1- \frac{\alpha\mu\mu_1}{2} - \frac{\omega}{n}}{1 - \gamma}; \frac{\mu}{2\gamma L}\rb$. Then for all $k \geq 0$:
         \begin{equation}
             \EE{Q}{\Psi_{k}} \leq \left(1 - \frac{\alpha \mu\mu_1}{2}\right)^{k+1}\Psi_0 + \left(\frac{\omega + 1}{2n} + \gamma c\right)\frac{2L \mu_2^2 \alpha}{\mu \mu_1}  \sigma^2.
         \end{equation}
    \end{theorem}
    \vspace{-4mm}
    Now we present the assumptions for nonconvex case.
    \vspace{-2mm}
    \begin{assumption}\label{assum:Lsmooth}
        The function $F$ is $L$-smooth.
    \end{assumption}
    \vspace{-4mm}
    \begin{assumption}\label{asymmetry}
    (Bounded data dissimilarity). There exists constant $ \zeta \geq 0$ such that $\forall x \in \R^d$ \\
    $
        \frac{1}{n} \sum_{i=1}^n \| \nabla f_i(x)  - \nabla F(x)\|^2_2 \leq \zeta^2. ~~
    $
    In particular, $\zeta = 0$, implies that all datasets stored in the $n$ devices are drawn from the same data distribution $\mathcal{D}$.
    \end{assumption}
    \vspace{-3mm}
    The following result shows that \algo converges in the nonconvex case.
    \vspace{-2mm}
    \begin{theorem}\label{thm:nonconvex}
      Suppose that Assumption \ref{assum:Lsmooth}, \ref{asymmetry} holds. Let $Q \in \mathcal{U}(\omega)$, Let $S = \{ w_0, w_1, \dots, w_{k-1} \}$ be generated using Algorithm \ref{alg:main}, and $\bar{w}$ be sampled uniformily at random from $S$, for $\alpha \leq \ \sqrt{\frac{n}{2Lw(w+1)\mu_2^2}}$ and $\gamma_k \leq  \frac{1 + \sqrt{1 - \frac{2 L \alpha^2 w(w+1)\mu_2^2}{n}} }{2(w+1)}$, and a parameter $c$ such as $c < \frac{\mu_1}{L \alpha \gamma_k } - \frac{\mu^2_2}{2\gamma_k} $  we have:
    \begin{align}
        \EE{Q}{\|\nabla F(\bar{w})\|_2^2}  &\leq 2 \frac{\bk^0}{ k\alpha \ls 2\mu_1 -  L \alpha \mu_2^2   - 2c L \alpha \gamma _k\rs} +  \frac{4 c L \alpha  }{  2\mu_1 -  L \alpha \mu_2^2   - 2c L \alpha \gamma _k}  \zeta^2 \nonumber\\&+  \frac{\mu_2^2+ 2 c }{  2\mu_1 - \ls L \alpha \mu_2^2  \rs - 2c L \alpha } L  \sigma^2 \nonumber\\
    \end{align}
with $\bk^k = F(w_k) - F^* + c \frac{L \alpha^2}{2}  \frac{1}{n} \sum_{i=1}^{n} \|h_k^i - h_*^i \|^2_2$
\end{theorem}
    
\begin{corollary}\label{cor:nonconvex}
Set $\gamma_k = \gamma$  , $\alpha = \frac{2 \mu_1 - 1}{L(\mu_2^2 + 2c \gamma) \sqrt{K}}$ and $h_0 = 0$, after $K$ iterations of algorithm \ref{alg:main}, in the nonconvex setting, the error $\epsilon$ is at worst $\frac{2}{\sqrt{K}} \frac{ L(\mu_2^2 + 2c \gamma)}{  \ls 2 \mu_1 - 1 \rs } \bk^0 + \frac{1}{\sqrt{K}} \frac{4 c \ls 2 \mu_1 - 1 \rs   }{ \mu_2^2 + 2c \gamma}  \zeta^2 + \frac{1}{\sqrt{K}} \frac{\ls\mu_2^2+ 2 c \gamma\rs \ls 2\mu_1 - 1\rs }{ \mu_2^2 + 2c \gamma }  \sigma^2$.
\end{corollary}

\section{Experiments}\label{sec:experiments}

We analyse the practical benefit of the proposed approach on regularized logistic regression problem for binary classification
\vspace{-3mm}
\begin{equation*}
    \min \limits_{w \in \R^d} \lb \tfrac{1}{n} \textstyle{\sum}_{i=1}^n \tfrac{1}{r} \textstyle{\sum}_{j=1}^r \log (1+ \exp(-b_{ij}a_{ij}^Tw)) + \frac{\mu}{2} \|w\|^2\rb,
\end{equation*}
\vspace{-0.5mm}
\noindent $\lb a_{ij}, b_{ij}\rb_{j \in [m]}$ are data points on $i$-th device. We use three datasets from the LIBSVM library \cite{chang2011libsvm}: gisette-scale ($5000$ features) and real-sim ($20958$ features), and a9a ($123$ features) (Appendix \ref{app:experiments}).

\textbf{\algo vs FLECS.} In this experiment we illustrate that gradient compression improves the convergence of FLECS in terms of communicated gradients per node. We use random dithering compressor for sketched Hessian and gradients with $128$ levels and $\infty$ norm. Hyperparameters of both methods are set the same: initial Hessian approximation $B_k^i = 0, ~ \omega = 10^{-5}, ~ \Omega=10^8,~ \alpha = 1,~ \beta = 1$,~ for FedSONIA update we set $\rho = \frac{1}{\Omega}$. For \algo $\gamma = 1$. We choose memory sizes $m = 1,~2,~4,~8$. Random dithering is used as compressor with $s=64$ levels and $p=\infty$-norm.  Both \algo and FLECS show their best performance with $m=1$.  Because of additional gradient compression, \algo outperforms FLECS in this low memory-size setup. 

According to the paper \cite{agafonov2022flecs} FLECS outperforms FedNL, DIANA and ADIANA. In the experiments (Figure \ref{fig:main}), we showed that gradient-compression improves FLECS convergence since it reduces communication complexity. Additional
experiments
can be found in Appendix \ref{app:experiments}.

\begin{figure}\label{fig:main}
\includegraphics[width=0.24\textwidth]{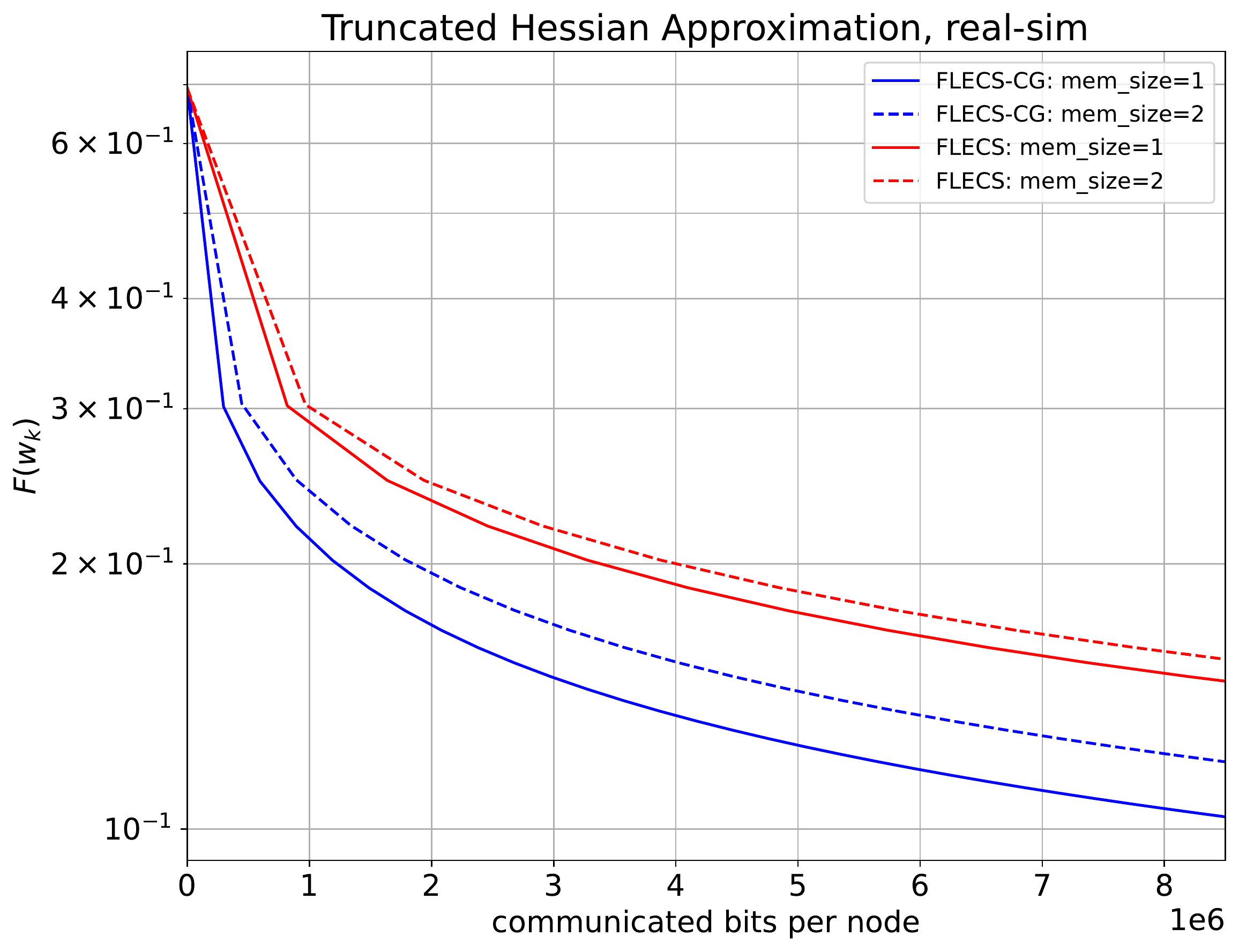}
\includegraphics[width=0.24\textwidth]{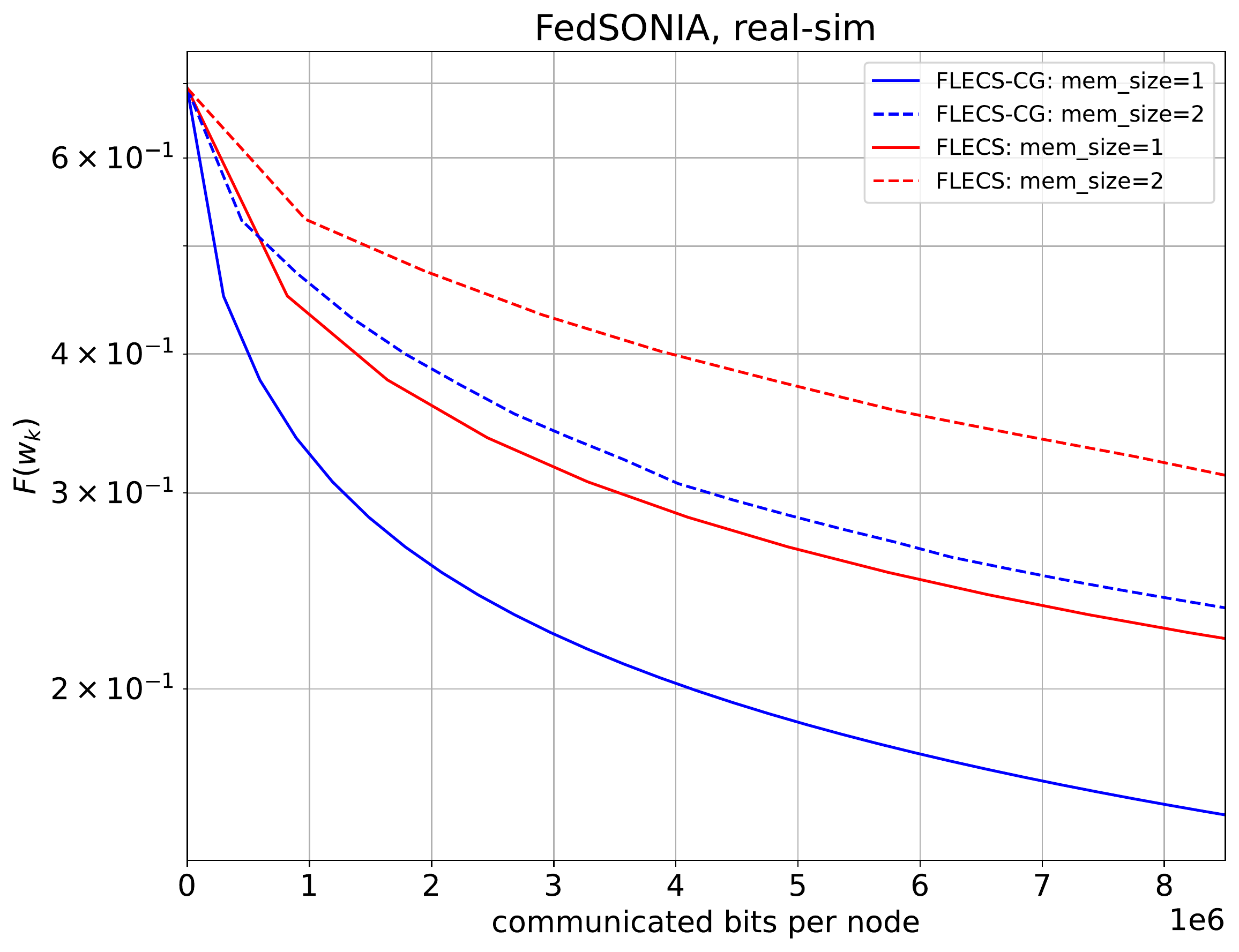}
\includegraphics[width=0.24\textwidth]{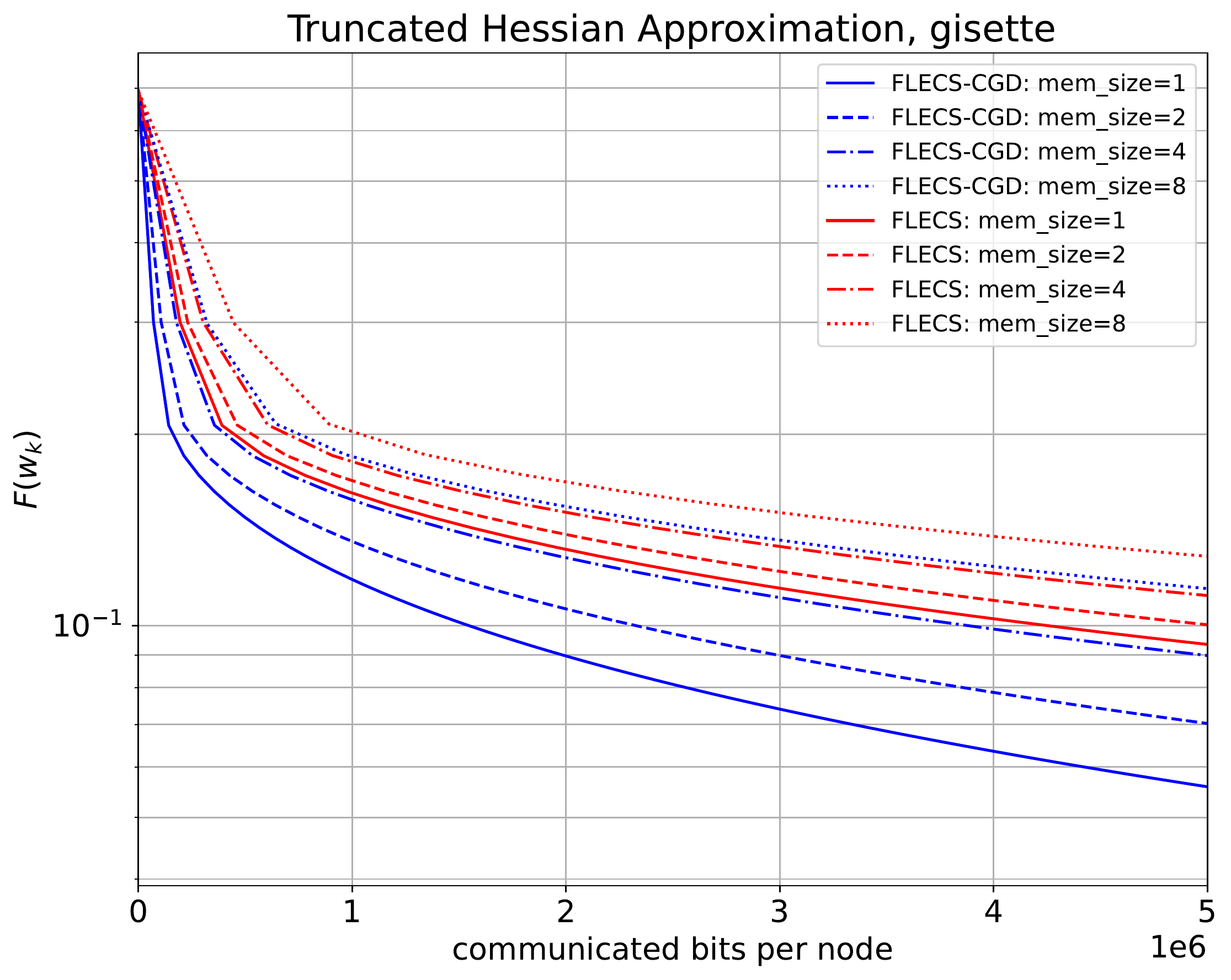}
\includegraphics[width=0.24\textwidth]{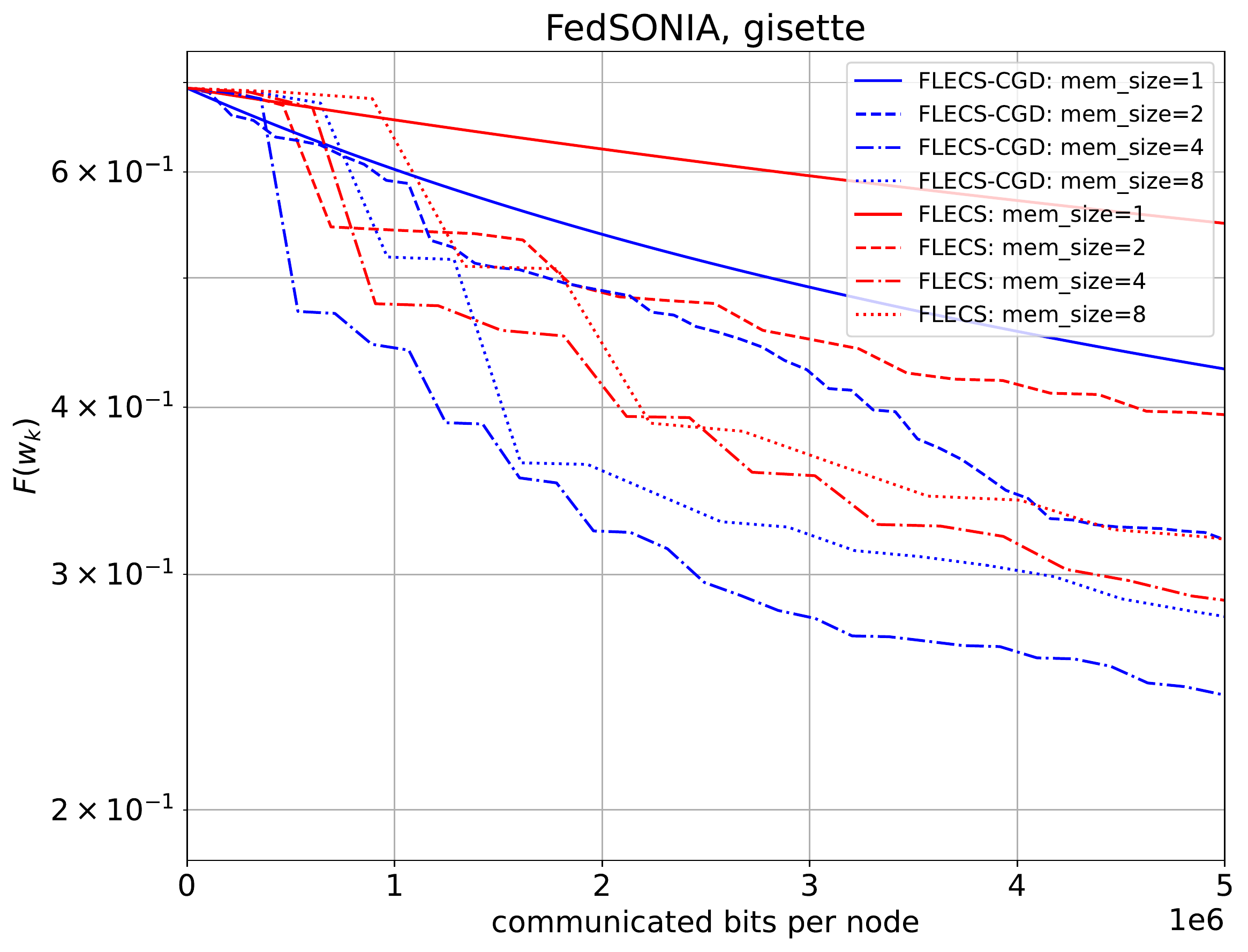}
\\
\includegraphics[width=0.24\textwidth]{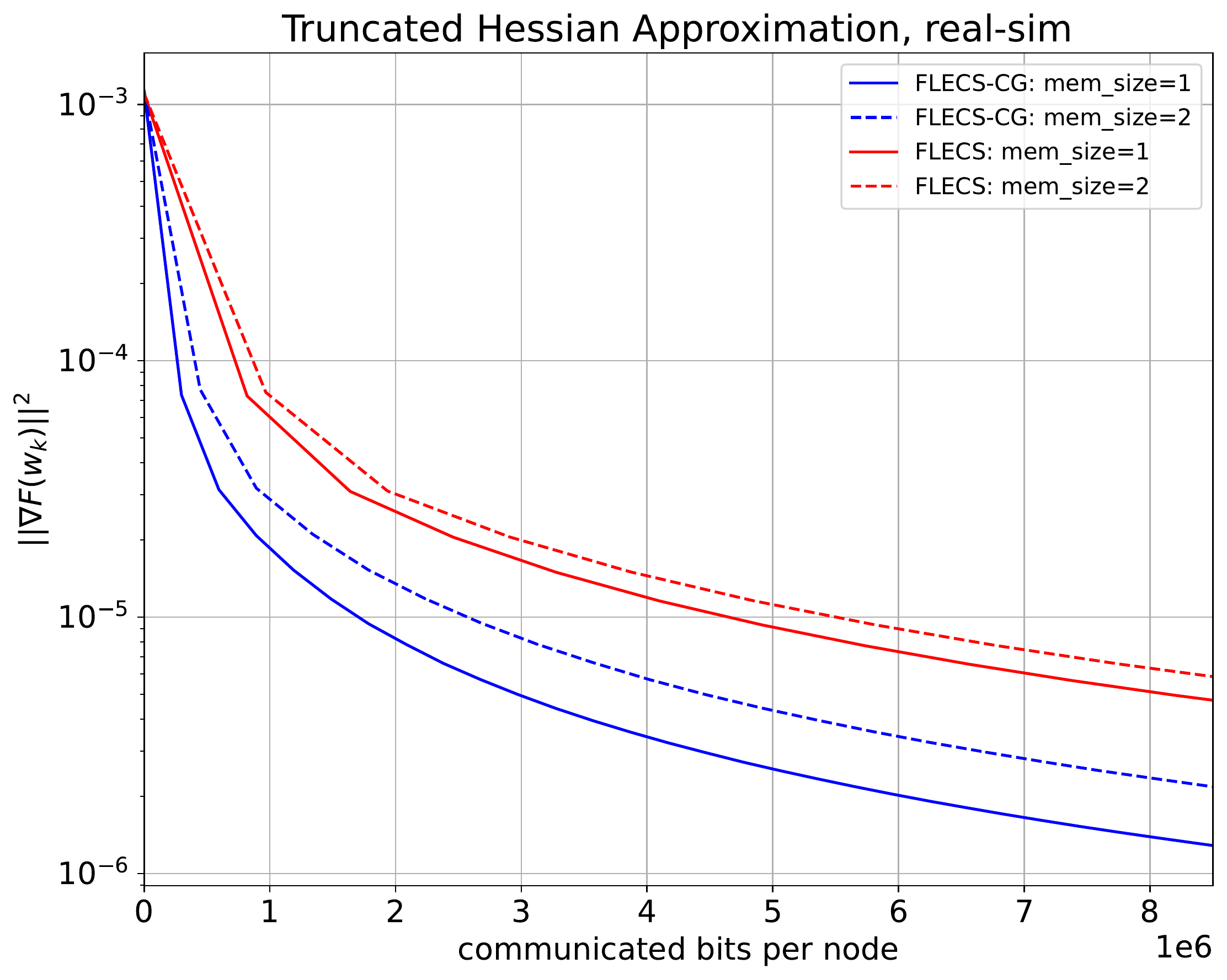}
\includegraphics[width=0.24\textwidth]{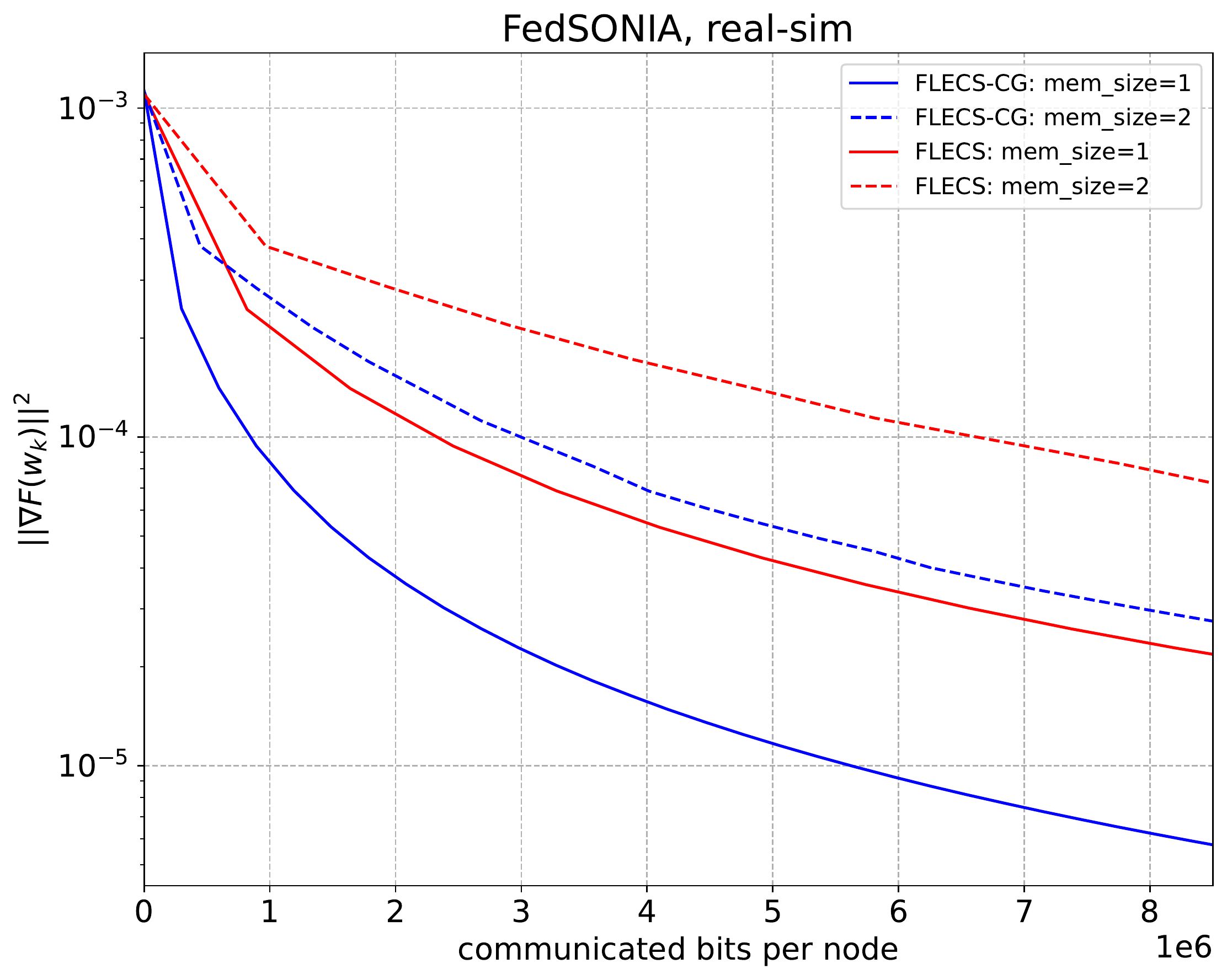}
\includegraphics[width=0.24\textwidth]{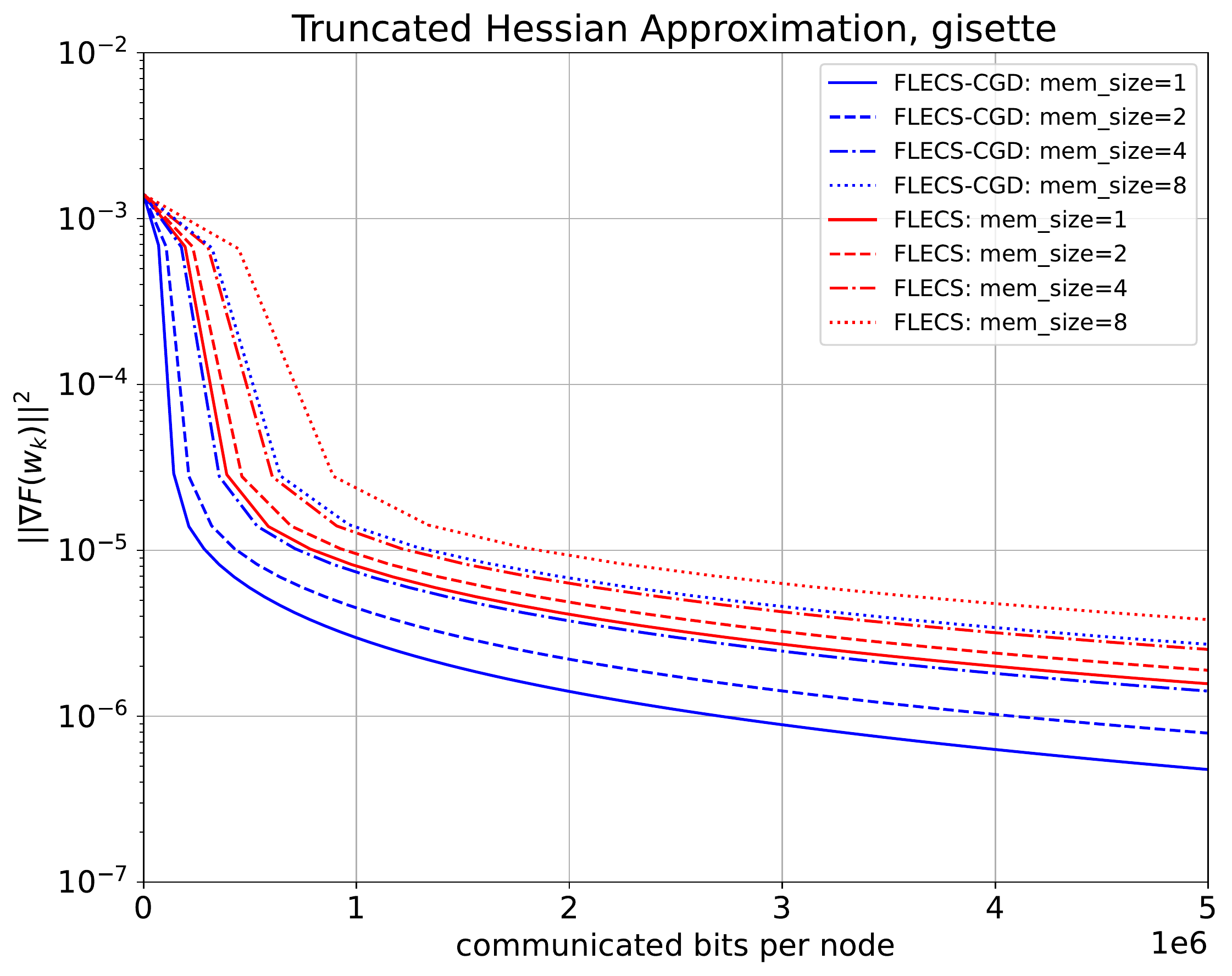}
\includegraphics[width=0.24\textwidth]{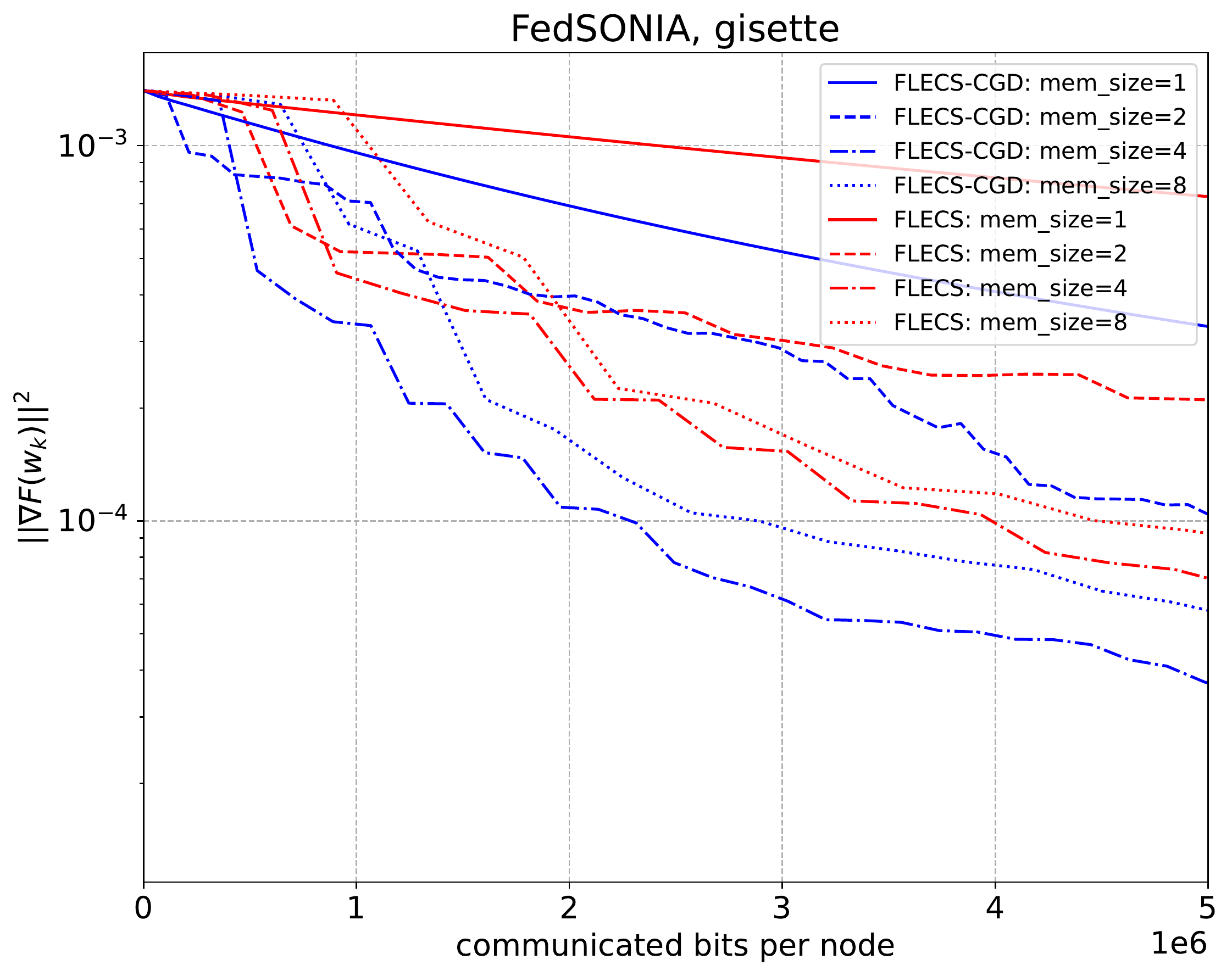}
\caption{Comparison of objective function $F(w_k)$ and the squared norm of gradient $\|\nabla F(w_k)\|^2$ for FLECS and \algo.}
\vskip-10pt
\end{figure}

\newpage

\bibliography{bibliography}

\begin{thebibliography}{40}
\providecommand{\natexlab}[1]{#1}
\providecommand{\url}[1]{\texttt{#1}}
\expandafter\ifx\csname urlstyle\endcsname\relax
  \providecommand{\doi}[1]{doi: #1}\else
  \providecommand{\doi}{doi: \begingroup \urlstyle{rm}\Url}\fi

\bibitem[Agafonov et~al.(2021)Agafonov, Dvurechensky, Scutari, Gasnikov,
  Kamzolov, Lukashevich, and Daneshmand]{agafonov21acc}
Artem Agafonov, Pavel Dvurechensky, Gesualdo Scutari, Alexander Gasnikov,
  Dmitry Kamzolov, Aleksandr Lukashevich, and Amir Daneshmand.
\newblock An accelerated second-order method for distributed stochastic
  optimization.
\newblock In \emph{2021 60th IEEE Conference on Decision and Control (CDC)},
  pages 2407--2413, 2021.
\newblock \doi{10.1109/CDC45484.2021.9683400}.

\bibitem[Agafonov et~al.(2022)Agafonov, Kamzolov, Tappenden, Gasnikov, and
  Tak{\'a}{\v{c}}]{agafonov2022flecs}
Artem Agafonov, Dmitry Kamzolov, Rachael Tappenden, Alexander Gasnikov, and
  Martin Tak{\'a}{\v{c}}.
\newblock Flecs: A federated learning second-order framework via compression
  and sketching.
\newblock \emph{arXiv preprint arXiv:2206.02009}, 2022.

\bibitem[Alistarh et~al.(2017)Alistarh, Grubic, Li, Tomioka, and
  Vojnovic]{alistarh2017qsgd}
Dan Alistarh, Demjan Grubic, Jerry Li, Ryota Tomioka, and Milan Vojnovic.
\newblock Qsgd: Communication-efficient sgd via gradient quantization and
  encoding.
\newblock \emph{Advances in Neural Information Processing Systems}, 30, 2017.

\bibitem[Bullins et~al.(2021)Bullins, Patel, Shamir, Srebro, and
  Woodworth]{bullins2021stochastic}
Brian Bullins, Kshitij Patel, Ohad Shamir, Nathan Srebro, and Blake~E
  Woodworth.
\newblock A stochastic newton algorithm for distributed convex optimization.
\newblock \emph{Advances in Neural Information Processing Systems}, 34, 2021.

\bibitem[Chang and Lin(2011)]{chang2011libsvm}
Chih-Chung Chang and Chih-Jen Lin.
\newblock Libsvm: a library for support vector machines.
\newblock \emph{ACM transactions on intelligent systems and technology (TIST)},
  2\penalty0 (3):\penalty0 1--27, 2011.

\bibitem[Chen et~al.(2022)Chen, Blum, Takac, and Sadler]{chen2022distributed}
Yicheng Chen, Rick~S Blum, Martin Takac, and Brian~M Sadler.
\newblock Distributed learning with sparsified gradient differences.
\newblock \emph{IEEE Journal of Selected Topics in Signal Processing}, 2022.

\bibitem[cn\' y et~al.(2016{\natexlab{a}})cn\' y, McMahan, Ramage, and
  Richtarik]{Konecny2016a}
Jakub~Kone\v cn\' y, H.~Brendan McMahan, Daniel Ramage, and Peter Richtarik.
\newblock Federated optimization: {D}istributed machine learning for on-device
  intelligence.
\newblock \emph{arXiv preprint arXiv:1610.02527}, 2016{\natexlab{a}}.

\bibitem[cn\' y et~al.(2016{\natexlab{b}})cn\' y, McMahan, Yu, Richtarik,
  Suresh, and Bacon]{Konecny2016b}
Jakub~Kone\v cn\' y, H.~Brendan McMahan, Felix~X. Yu, Peter Richtarik,
  Ananda~Theertha Suresh, and Dave Bacon.
\newblock Federated learning: Strategies for improving communication
  efficiency.
\newblock In \emph{NIPS Workshop on Private Multi-Party Machine Learning},
  2016{\natexlab{b}}.
\newblock URL \url{https://arxiv.org/abs/1610.05492}.

\bibitem[Daneshmand et~al.(2021)Daneshmand, Scutari, Dvurechensky, and
  Gasnikov]{daneshmand2021newton}
Amir Daneshmand, Gesualdo Scutari, Pavel Dvurechensky, and Alexander Gasnikov.
\newblock Newton method over networks is fast up to the statistical precision,
  2021.

\bibitem[Dvurechenskii et~al.(2018)Dvurechenskii, Dvinskikh, Gasnikov, Uribe,
  and Nedich]{dvurechenskii2018decentralize}
Pavel Dvurechenskii, Darina Dvinskikh, Alexander Gasnikov, Cesar Uribe, and
  Angelia Nedich.
\newblock Decentralize and randomize: Faster algorithm for wasserstein
  barycenters.
\newblock In \emph{Advances in Neural Information Processing Systems}, pages
  10760--10770, 2018.

\bibitem[Dvurechensky et~al.(2021)Dvurechensky, Kamzolov, Lukashevich, Lee,
  Ordentlich, Uribe, and Gasnikov]{dvurechensky2021hyperfast}
Pavel Dvurechensky, Dmitry Kamzolov, Aleksandr Lukashevich, Soomin Lee, Erik
  Ordentlich, C{\'e}sar~A Uribe, and Alexander Gasnikov.
\newblock Hyperfast second-order local solvers for efficient statistically
  preconditioned distributed optimization.
\newblock \emph{arXiv preprint arXiv:2102.08246}, 2021.

\bibitem[Hard et~al.(2018)Hard, Rao, Mathews, Ramaswamy, Beaufays, Augenstein,
  Eichner, Kiddon, and Ramage]{https://doi.org/10.48550/arxiv.1811.03604}
Andrew Hard, Kanishka Rao, Rajiv Mathews, Swaroop Ramaswamy, Françoise
  Beaufays, Sean Augenstein, Hubert Eichner, Chloé Kiddon, and Daniel Ramage.
\newblock Federated learning for mobile keyboard prediction, 2018.
\newblock URL \url{https://arxiv.org/abs/1811.03604}.

\bibitem[Horv{\'a}th et~al.(2019{\natexlab{a}})Horv{\'a}th, Ho, Horvath, Sahu,
  Canini, and Richt{\'a}rik]{horvath2019natural}
Samuel Horv{\'a}th, Chen-Yu Ho, Ludovit Horvath, Atal~Narayan Sahu, Marco
  Canini, and Peter Richt{\'a}rik.
\newblock Natural compression for distributed deep learning.
\newblock \emph{arXiv preprint arXiv:1905.10988}, 2019{\natexlab{a}}.

\bibitem[Horv{\'a}th et~al.(2019{\natexlab{b}})Horv{\'a}th, Kovalev,
  Mishchenko, Stich, and Richt{\'a}rik]{horvath2019stochastic}
Samuel Horv{\'a}th, Dmitry Kovalev, Konstantin Mishchenko, Sebastian Stich, and
  Peter Richt{\'a}rik.
\newblock Stochastic distributed learning with gradient quantization and
  variance reduction.
\newblock \emph{arXiv preprint arXiv:1904.05115}, 2019{\natexlab{b}}.

\bibitem[Horváth et~al.(2019)Horváth, Kovalev, Mishchenko, Stich, and
  Richtárik]{Horvath2019}
Samuel Horváth, Dmitry Kovalev, Konstantin Mishchenko, Sebastian Stich, and
  Peter Richtárik.
\newblock Stochastic distributed learning with gradient quantization and
  variance reduction, 2019.

\bibitem[Jaggi et~al.(2014)Jaggi, Smith, Tak{\'a}c, Terhorst, Krishnan,
  Hofmann, and Jordan]{jaggi2014communication}
Martin Jaggi, Virginia Smith, Martin Tak{\'a}c, Jonathan Terhorst, Sanjay
  Krishnan, Thomas Hofmann, and Michael~I Jordan.
\newblock Communication-efficient distributed dual coordinate ascent.
\newblock \emph{Advances in neural information processing systems}, 27, 2014.

\bibitem[Kraska et~al.(2013)Kraska, Talwalkar, Duchi, Griffith, Franklin, and
  Jordan]{kra13}
Tim Kraska, Ameet Talwalkar, John~C Duchi, Rean Griffith, Michael~J Franklin,
  and Michael~I Jordan.
\newblock Mlbase: A distributed machine-learning system.
\newblock In \emph{CIDR}, volume~1, pages 2--1, 2013.

\bibitem[Li et~al.(2020)Li, Fan, Tse, and Lin]{LI2020106854}
Li~Li, Yuxi Fan, Mike Tse, and Kuo-Yi Lin.
\newblock A review of applications in federated learning.
\newblock \emph{Computers \& Industrial Engineering}, 149:\penalty0 106854,
  2020.
\newblock ISSN 0360-8352.
\newblock \doi{https://doi.org/10.1016/j.cie.2020.106854}.
\newblock URL
  \url{https://www.sciencedirect.com/science/article/pii/S0360835220305532}.

\bibitem[Li et~al.(2014)Li, Andersen, Park, Smola, Ahmed, Josifovski, Long,
  Shekita, and Su]{li2014scaling}
Mu~Li, David~G Andersen, Jun~Woo Park, Alexander~J Smola, Amr Ahmed, Vanja
  Josifovski, James Long, Eugene~J Shekita, and Bor-Yiing Su.
\newblock Scaling distributed machine learning with the parameter server.
\newblock In \emph{11th USENIX Symposium on Operating Systems Design and
  Implementation (OSDI 14)}, pages 583--598, 2014.

\bibitem[Ma et~al.(2017)Ma, Kone{\v{c}}n{\`y}, Jaggi, Smith, Jordan,
  Richt{\'a}rik, and Tak{\'a}{\v{c}}]{ma2017distributed}
Chenxin Ma, Jakub Kone{\v{c}}n{\`y}, Martin Jaggi, Virginia Smith, Michael~I
  Jordan, Peter Richt{\'a}rik, and Martin Tak{\'a}{\v{c}}.
\newblock Distributed optimization with arbitrary local solvers.
\newblock \emph{optimization Methods and Software}, 32\penalty0 (4):\penalty0
  813--848, 2017.

\bibitem[Marecek et~al.(2014)Marecek, Richt{\'a}rik, and
  Takac]{marecek2014distributed}
Jakub Marecek, Peter Richt{\'a}rik, and Martin Takac.
\newblock Distributed block coordinate descent for minimizing partially
  separable functions.
\newblock \emph{Numerical Analysis and Optimization 2014, Springer Proceedings
  in Mathematics and Statistics}, 2014.

\bibitem[McMahan et~al.(2016)McMahan, Moore, Ramage, and
  y~Arcas]{DBLP:journals/corr/McMahanMRA16}
H.~Brendan McMahan, Eider Moore, Daniel Ramage, and Blaise~Ag{\"{u}}era
  y~Arcas.
\newblock Federated learning of deep networks using model averaging.
\newblock \emph{CoRR}, abs/1602.05629, 2016.
\newblock URL \url{http://arxiv.org/abs/1602.05629}.

\bibitem[McMahan et~al.(2017)McMahan, Moore, Ramage, Hampson, and
  y~Arcas]{McMahan2017}
H~Brendan McMahan, Eider Moore, Daniel Ramage, Seth Hampson, and Blaise~Agüera
  y~Arcas.
\newblock Communication-efficient learning of deep networks from decentralized
  data.
\newblock In \emph{In Proceedings of the 20th International Conference on
  Artificial Intelligence and Statistics (AISTATS)}, 2017.

\bibitem[Mishchenko et~al.(2019{\natexlab{a}})Mishchenko, Gorbunov,
  Tak{\'a}{\v{c}}, and Richt{\'a}rik]{mishchenko2019distributed}
Konstantin Mishchenko, Eduard Gorbunov, Martin Tak{\'a}{\v{c}}, and Peter
  Richt{\'a}rik.
\newblock Distributed learning with compressed gradient differences.
\newblock \emph{arXiv preprint arXiv:1901.09269}, 2019{\natexlab{a}}.

\bibitem[Mishchenko et~al.(2019{\natexlab{b}})Mishchenko, Gorbunov, {Takáč},
  , and {Richtárik}]{Mishchenko2019}
Konstantin Mishchenko, Eduard Gorbunov, Martin {Takáč}, , and Peter
  {Richtárik}.
\newblock Distributed learning with compressed gradient differences,
  2019{\natexlab{b}}.

\bibitem[Mousavi et~al.(2019)Mousavi, Nazari, Tak{\'a}{\v{c}}, and
  Motee]{mousavi2019multi}
Hossein~K Mousavi, Mohammadreza Nazari, Martin Tak{\'a}{\v{c}}, and Nader
  Motee.
\newblock Multi-agent image classification via reinforcement learning.
\newblock In \emph{2019 IEEE/RSJ International Conference on Intelligent Robots
  and Systems (IROS)}, pages 5020--5027. IEEE, 2019.

\bibitem[Nedi{\'c} et~al.(2017)Nedi{\'c}, Olshevsky, and Uribe]{nedic2017fast}
Angelia Nedi{\'c}, Alex Olshevsky, and C{\'e}sar~A Uribe.
\newblock Fast convergence rates for distributed non-bayesian learning.
\newblock \emph{IEEE Transactions on Automatic Control}, 62\penalty0
  (11):\penalty0 5538--5553, 2017.

\bibitem[Nguyen et~al.(2017)Nguyen, Liu, Scheinberg, and
  Tak{\'a}{\v{c}}]{nguyen2017sarah}
Lam Nguyen, Jie Liu, Katya Scheinberg, and Martin Tak{\'a}{\v{c}}.
\newblock Sarah: A novel method for machine learning problems using stochastic
  recursive gradient.
\newblock In \emph{In 34th International Conference on Machine Learning, ICML
  2017}, 2017.

\bibitem[Nguyen et~al.(2021)Nguyen, Scheinberg, and
  Tak{\'a}{\v{c}}]{nguyen2021inexact}
Lam~M Nguyen, Katya Scheinberg, and Martin Tak{\'a}{\v{c}}.
\newblock Inexact sarah algorithm for stochastic optimization.
\newblock \emph{Optimization Methods and Software}, 36\penalty0 (1):\penalty0
  237--258, 2021.

\bibitem[Paternain et~al.(2019)Paternain, Mokhtari, and
  Ribeiro]{paternain2019newton}
Santiago Paternain, Aryan Mokhtari, and Alejandro Ribeiro.
\newblock A newton-based method for nonconvex optimization with fast evasion of
  saddle points.
\newblock \emph{SIAM Journal on Optimization}, 29\penalty0 (1):\penalty0
  343--368, 2019.

\bibitem[Rabbat and Nowak(2004)]{rab04}
M.G. Rabbat and R.D. Nowak.
\newblock Decentralized source localization and tracking wireless sensor
  networks.
\newblock In \emph{Proceedings of the IEEE International Conference on
  Acoustics, Speech, and Signal Processing}, volume~3, pages 921--924, 2004.

\bibitem[Ram et~al.(2009)Ram, Veeravalli, and Nedic]{ram2009distributed}
Sundhar~Srinivasan Ram, Venugopal~V Veeravalli, and Angelia Nedic.
\newblock Distributed non-autonomous power control through distributed convex
  optimization.
\newblock In \emph{IEEE INFOCOM 2009}, pages 3001--3005. IEEE, 2009.

\bibitem[Richt{\'a}rik and Tak{\'a}c(2016)]{richtarik2016distributed}
Peter Richt{\'a}rik and Martin Tak{\'a}c.
\newblock Distributed coordinate descent method for learning with big data.
\newblock \emph{Journal of Machine Learning Research}, 17:\penalty0 1--25,
  2016.

\bibitem[Safaryan et~al.(2021)Safaryan, Islamov, Qian, and
  Richt{\'a}rik]{safaryan2021fednl}
Mher Safaryan, Rustem Islamov, Xun Qian, and Peter Richt{\'a}rik.
\newblock Fednl: Making newton-type methods applicable to federated learning.
\newblock \emph{arXiv preprint arXiv:2106.02969}, 2021.

\bibitem[Shi et~al.(2021)Shi, Loizou, Richt{\'a}rik, and
  Tak{\'a}{\v{c}}]{shi2021ai}
Zheng Shi, Nicolas Loizou, Peter Richt{\'a}rik, and Martin Tak{\'a}{\v{c}}.
\newblock Ai-sarah: Adaptive and implicit stochastic recursive gradient
  methods.
\newblock \emph{arXiv preprint arXiv:2102.09700}, 2021.

\bibitem[Smith et~al.(2018)Smith, Forte, Chenxin, Tak{\'a}{\v{c}}, Jordan, and
  Jaggi]{smith2018cocoa}
Virginia Smith, Simone Forte, Ma~Chenxin, Martin Tak{\'a}{\v{c}}, Michael~I
  Jordan, and Martin Jaggi.
\newblock Cocoa: A general framework for communication-efficient distributed
  optimization.
\newblock \emph{Journal of Machine Learning Research}, 18:\penalty0 230, 2018.

\bibitem[Uribe et~al.(2018)Uribe, Dvinskikh, Dvurechensky, Gasnikov, and
  Nedi\'c]{uribe2018distributed}
C\'esar~A. Uribe, Darina Dvinskikh, Pavel Dvurechensky, Alexander Gasnikov, and
  Angelia Nedi\'c.
\newblock Distributed computation of {W}asserstein barycenters over networks.
\newblock In \emph{2018 IEEE 57th Annual Conference on Decision and Control
  (CDC)}, 2018.
\newblock Accepted, arXiv:1803.02933.

\bibitem[Xiao and Boyd(2006)]{xia06}
Lin Xiao and Stephen Boyd.
\newblock {Optimal scaling of a gradient method for distributed resource
  allocation}.
\newblock \emph{Journal of Optimization Theory and Applications}, 129\penalty0
  (3):\penalty0 469--488, 2006.

\bibitem[Zhang and Xiao(2018)]{Zhang2018}
Yuchen Zhang and Lin Xiao.
\newblock \emph{Communication-Efficient Distributed Optimization of
  Self-concordant Empirical Loss}, pages 289--341.
\newblock Springer International Publishing, Cham, 2018.
\newblock ISBN 978-3-319-97478-1.
\newblock \doi{10.1007/978-3-319-97478-1_11}.

\bibitem[Zhize~Li and Richtárik(2020)]{Li2020accFL}
Xun~Qian Zhize~Li, Dmitry~Kovalev and Peter Richtárik.
\newblock Acceleration for compressed gradient descent in distributed and
  federated optimization.
\newblock \emph{International Conference on Machine Learning}, 37, 2020.

\end{thebibliography}

\newpage

\appendix

\section{FLECS-CGD}

\subsection{Hessian Approximation Update}
    \begin{algorithm}[t]
        \caption{Truncated L-SR1 update\cite{agafonov2022flecs}}\label{alg:hess_lsr1}
        \begin{algorithmic}[1]
            \REQUIRE $\Y_k^i \in \R^{d \times m}, ~M_k^i \in \R^{m \times m}, ~B_k^i \in \R^{d \times d}, ~S_k \in \R^{d \times m}$ for $i = 1, \ldots, n$,
            $~\omega > 0$ -- truncation constant.
            \STATE \textbf{On the server:}
            \FOR{$i = 1, \ldots, n$}
            \STATE compute $(M_k^i - (S_k^i)^T\Y_k^i) =U_k^i L_k^i (U_k^i)^T$;
            \STATE truncate $(L_k^i)^{-1}$  
            to form $[(L_k^i)^{-1}]_\omega$;
            \STATE compute $B_{k+1}^i$ via \begin{equation}\label{eq:LSR1_truncated}
                B_{k+1} = B_k + (\bY_k - B_k^iS_k)U_k^i [(L_k^i)^{-1}]_\omega  (U_k^i)^T(\bY_k - B_k^iS_k)^T.
            \end{equation}
            \ENDFOR
        \end{algorithmic}
    \end{algorithm}
    
    \begin{algorithm}[h]
      \caption{Direct update \cite{agafonov2022flecs}}\label{alg:hess_direct}
        \begin{algorithmic}[1]
            \REQUIRE $\Y_k^i \in \R^{d \times m}, ~M_k^i \in \R^{m \times m}, ~B_k^i \in \R^{d \times d}$ $\forall i$
            \STATE \textbf{On the server:}
            $0<\beta_k\leq 1$ -- learning rate.
            \FOR{$i = 1, \ldots, n$}
            \STATE compute $\B^i_k = \Y^i_k (M^i_k)^\dagger (\Y^i_k)^T$;
            \STATE select learning rate $\beta_k$ 
            \STATE compute
            $B^i_{k+1} = (1 - \beta_k)B^i_{k} + \beta_k\B^i_k$.
            \ENDFOR
        \end{algorithmic}
    \end{algorithm}

    \subsection{Iterate update}
    
    \begin{definition}\label{def:truncation}
        Let $B_k, V_k, \Lambda_k$ be matrices such that $B_k = V_k \Lambda_k V_k^T$, and let $0 < \omega \leq \Omega$. The truncated inverse Hessian approximation of $B_{k}$ is $\left(|B_{k}|_\omega^\Omega\right)^{-1} := V_k (|\Lambda_k|_\omega^\Omega)^{-1}V_k^T$,
        where \\ $(|\Lambda_k|_\omega^\Omega)_{ii} = \min\lb \max \lb |\Lambda|_{ii}, \omega\rb, \Omega\rb.$
    \end{definition}
    Definition \ref{def:truncation} was proposed in \cite{paternain2019newton} and was used to provide a convergence guarantee for their Nonconvex Newton method (to a local minimum). Firstly, an eigen-decomposition of $B_k$ is computed, but with every eigenvalue replaced by its absolute value. Secondly, a thresholding step is applied, so that any eigenvalue (in absolute value) that is smaller (resp. greater) than a user defined threshold $\omega$ (resp. $\Omega$) is replaced by $\omega$ (resp. $\Omega$).
    
    \begin{algorithm}[t]
           \caption{Truncated inverse Hessian approximation \cite{agafonov2022flecs}}\label{alg:approx_trunc}
            \begin{algorithmic}[1]
            \REQUIRE $\nabla F(w_k) \in \R^d, ~ \bY_k \in \R^{d\times m}, ~ M_k \in \R^{m\times m}, ~ B_{k+1} \in \R^{d\times d}$,
            $\Omega > \omega > 0$ -- truncation constants.
            \STATE \textbf{On the server:}
            \STATE compute spectral decomposition $B_{k+1} = V_k \Lambda_k V_k^T$;
            \STATE truncate $\Lambda_k$ 
            to form $|\Lambda_k|_\omega^\Omega$  via Definition \ref{def:truncation};
            \STATE compute search direction $p_k$ via 
            $p_k = \left(|B_{k+1}|_\omega^\Omega\right)^{-1} \nabla F(w_k);$
            \RETURN $p_k$.
        \end{algorithmic}
    \end{algorithm}
    
    \begin{algorithm}[t]
           \caption{FedSONIA \cite{agafonov2022flecs}}\label{alg:approx_fedsonia}
            \begin{algorithmic}[1]
            \REQUIRE $\nabla F(w_k) \in \R^d, ~ \bY_k \in \R^{d\times m}, ~ M_k \in \R^{m\times m}$,
            $\Omega > \omega > 0$ -- truncation constants.
            \STATE \textbf{On the server:}
            \STATE compute $B_k^{\text{SONIA}} \eqdef \Y_k (M_k)^\dagger \Y_k^T$;
            \STATE compute $QR$ factorization of $\Y_k (= Q_k R_k)$;
            \STATE compute spectral decomposition of $R_k(M_k)^\dagger R_k^T (= V_k \Lambda_k V_k^T)$;
            \STATE construct $\widetilde{V}_k := Q_k V_k$;
            \STATE truncate $\Lambda_k$ 
            to form $|\Lambda_k|_\omega^\Omega$ via Definition \ref{def:truncation};
            \STATE Set $\rho_k$ and decompose gradient via 
            $\nabla F(w_k) = g_k + g_k^\perp;$
            \STATE Compute search direction $p_k$ via 
            $p_k :=  - \left(|B_{k+1}^{\text{SONIA}}|_\omega^\Omega\right)^{-1} g_k - \rho_k g^\perp_k,
;$
            \RETURN $p_k$.
        \end{algorithmic}
    \end{algorithm}

\section{Additional Experiments}\label{app:experiments}
Comparison between FLECS and \algo is provided on Figure \ref{fig:add_exp}.
\\
Comparison between \algo's iterate updates (Algorithms \ref{alg:approx_trunc}, \ref{alg:approx_fedsonia}) is provided on Figure \ref{fig:add_exp_iters}.

\begin{figure}
\includegraphics[width=0.46\textwidth]{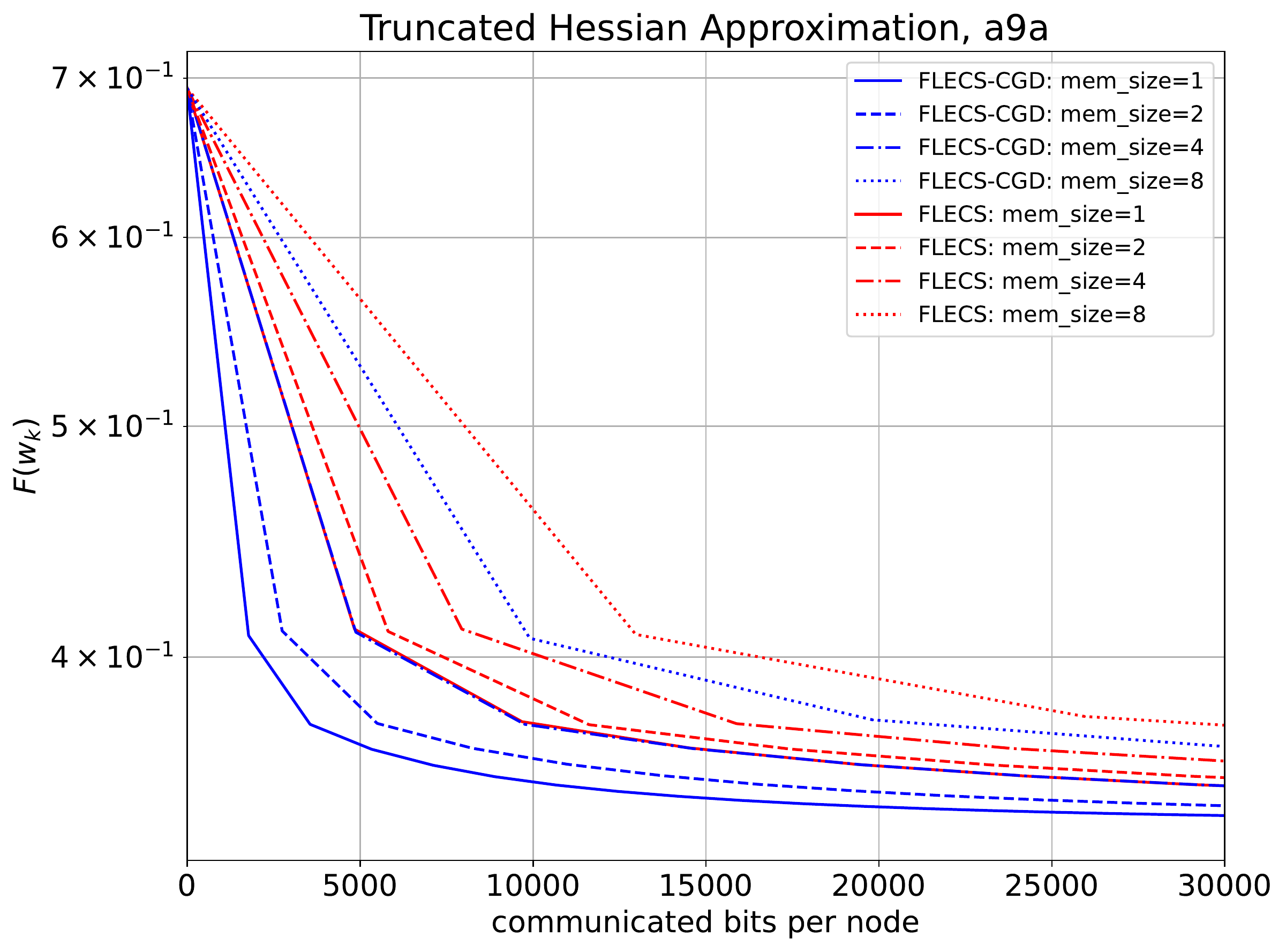}
\includegraphics[width=0.46\textwidth]{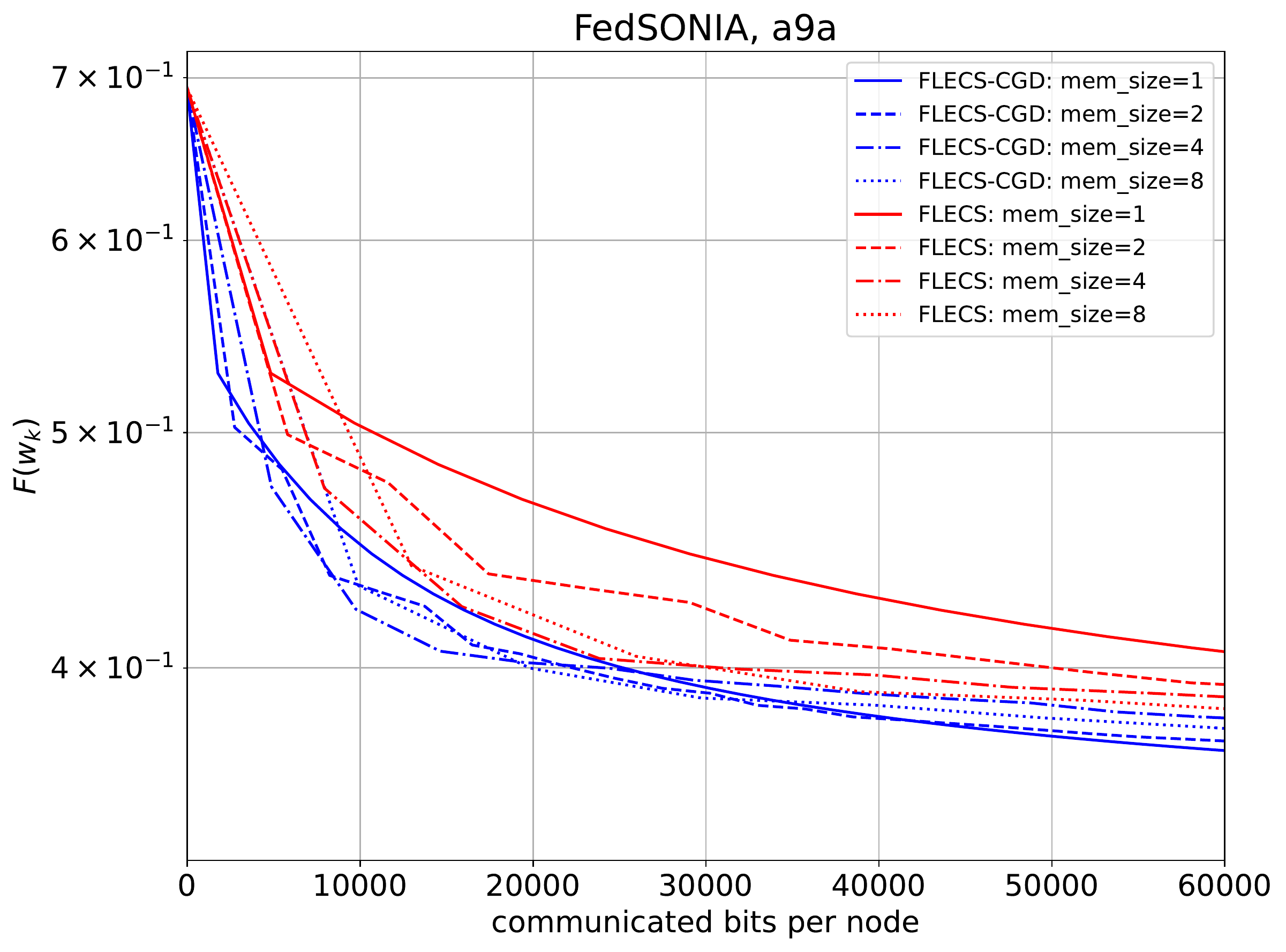}
\\
\includegraphics[width=0.46\textwidth]{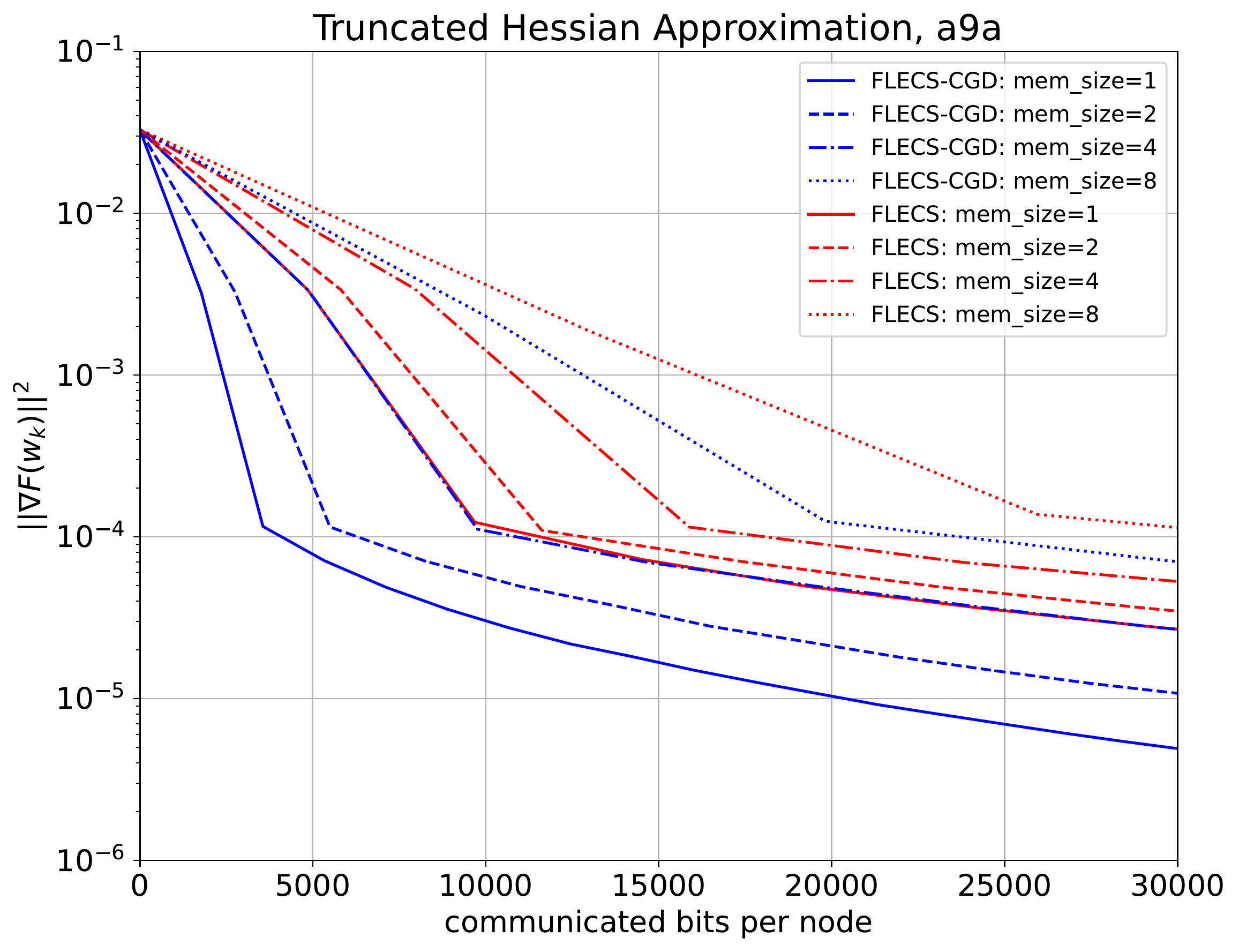}
\includegraphics[width=0.46\textwidth]{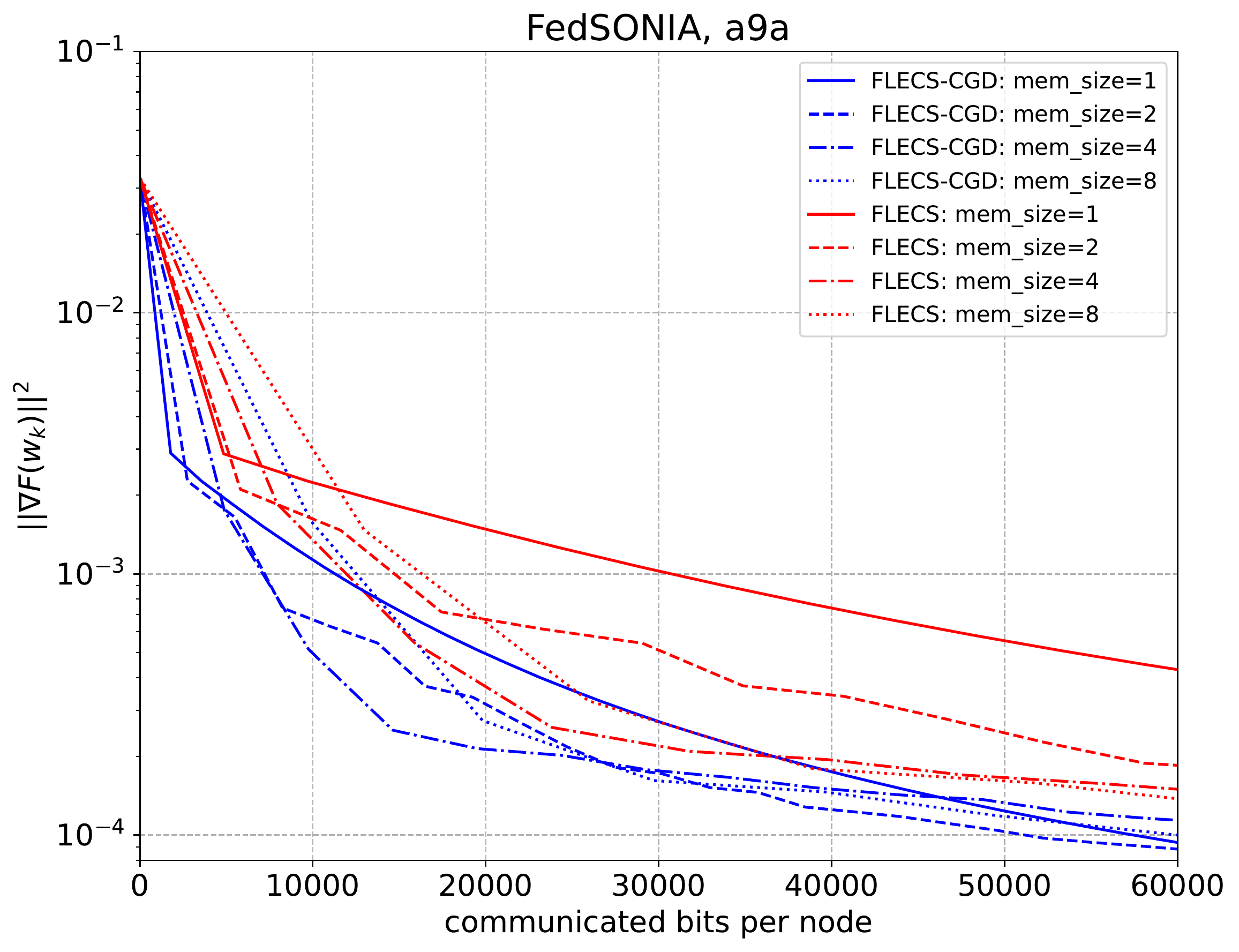}
\caption{Comparison of objective function $F(w_k)$ and the squared norm of gradient $\|\nabla F(w_k)\|^2$ for FLECS and \algo.} \label{fig:add_exp}
\vskip-10pt
\end{figure}

\begin{figure}
\includegraphics[width=0.33\textwidth]{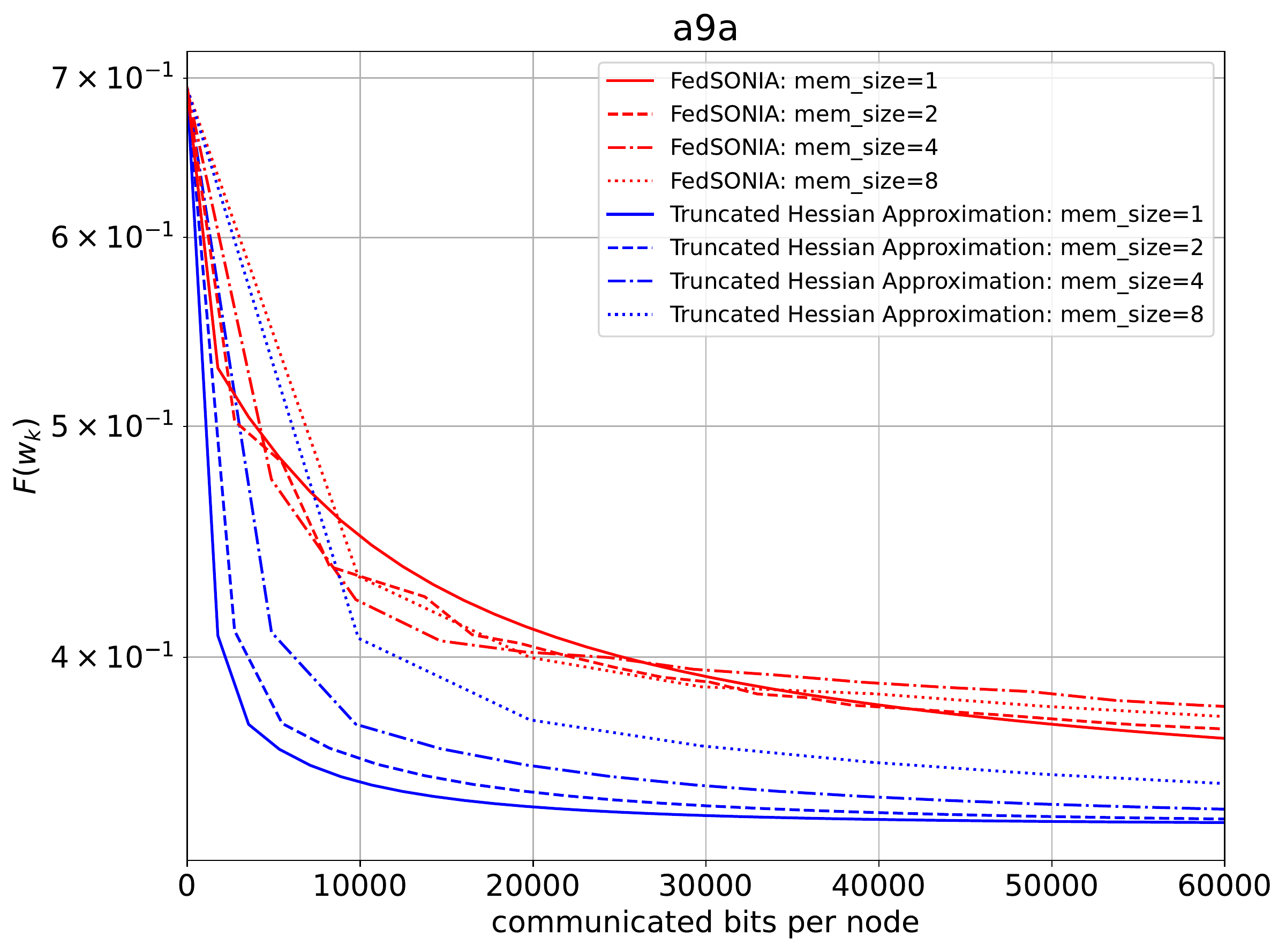}
\includegraphics[width=0.33\textwidth]{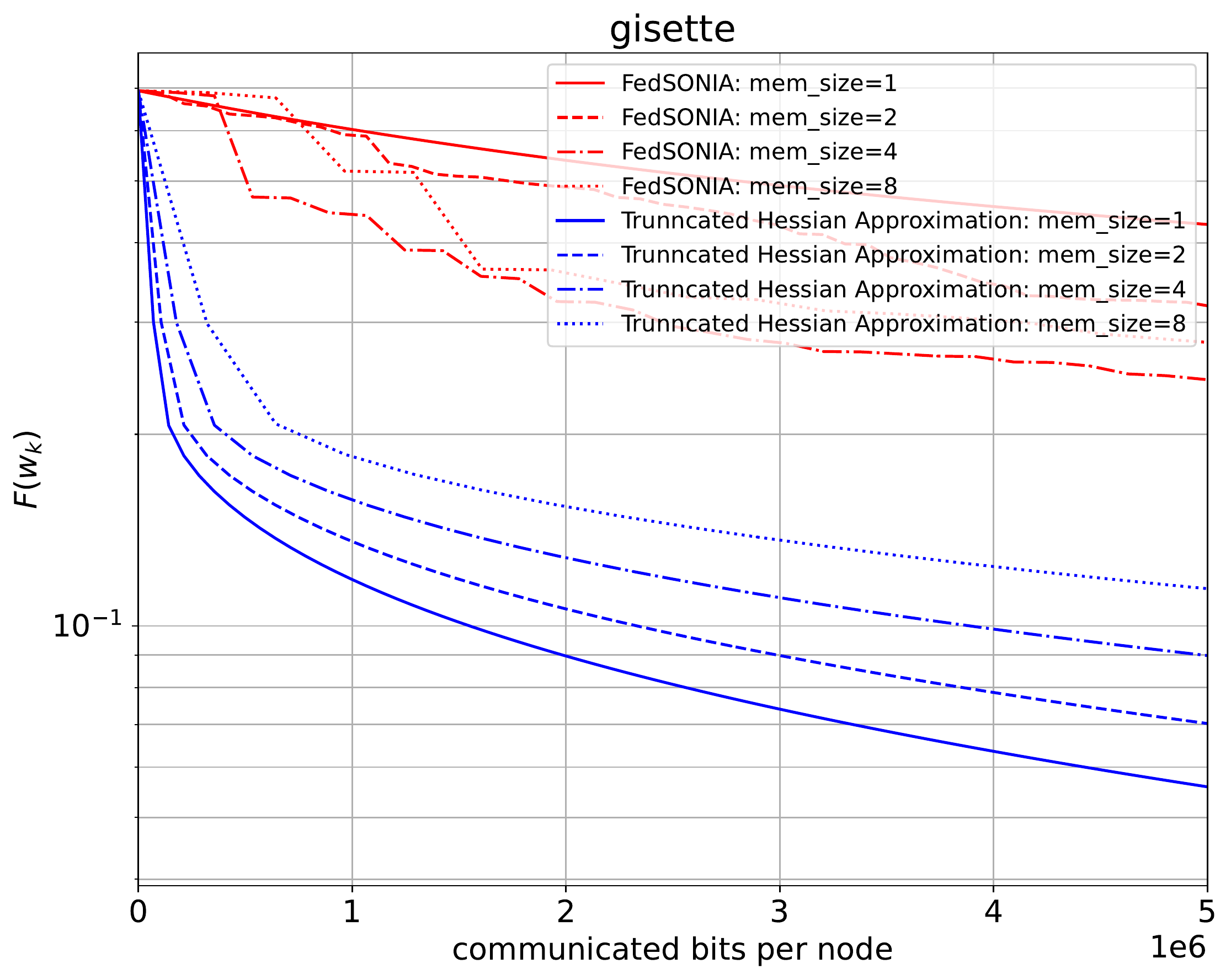}
\includegraphics[width=0.33\textwidth]{experiments/flecscg_vs_flecs/real-sim/FLECS_THA_FEDSONIA.pdf}
\\
\includegraphics[width=0.33\textwidth]{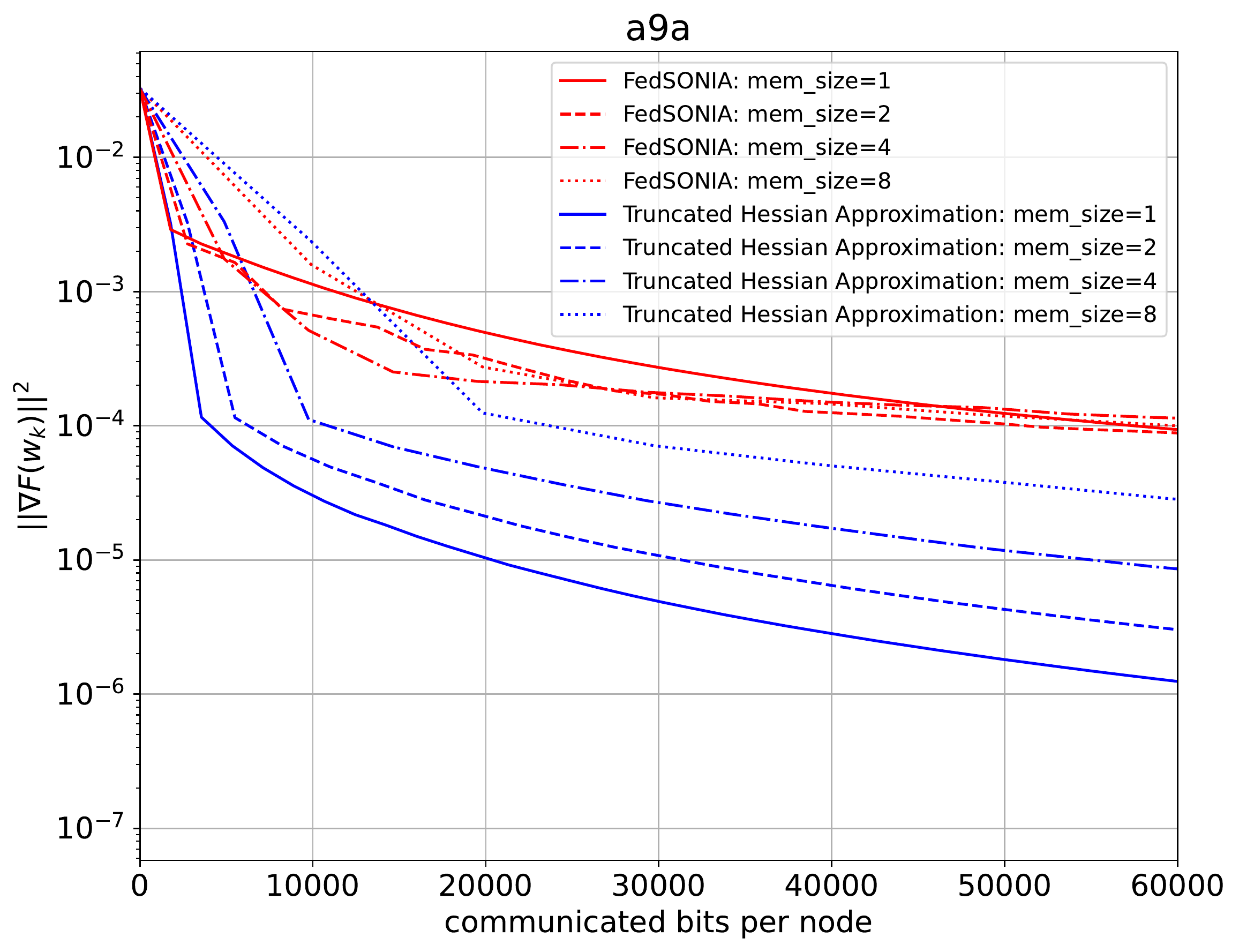}
\includegraphics[width=0.33\textwidth]{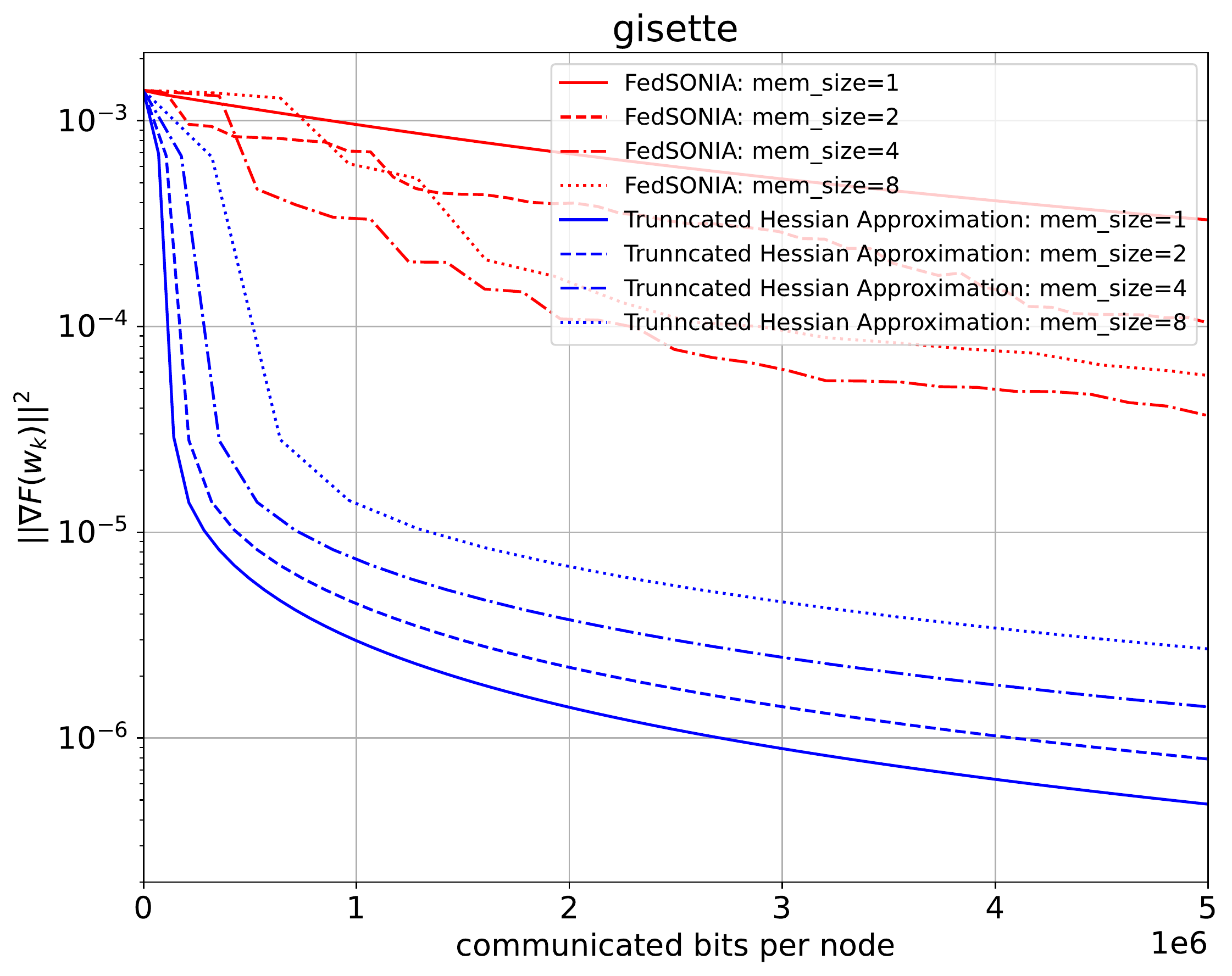}
\includegraphics[width=0.33\textwidth]{experiments/flecscg_vs_flecs/real-sim/FLECS_THA_FEDSONIA_grad.pdf}
\caption{Comparison of objective function $F(w_k)$ and the squared norm of gradient $\|\nabla F(w_k)\|^2$ for different iterate updates in \algo.} \label{fig:add_exp_iters}
\vskip-10pt
\end{figure}

\section{Proofs}
\subsection{Basic identities and propertites}

Let $x,y \in \R^d$ and $\alpha \in [0,1]$:
\begin{align}
    \| \alpha x + (1 - \alpha)y\|^2_2 = \alpha\|x \|_2^2 + (1 - \alpha)\|y \|_2^2 - \alpha(1 - \alpha)\|x-y \|_2^2 \label{eq:norm_convex_comb}
\end{align}

Let $g$ be a random vector, and $h \in \R^d$:
\begin{align}
    \EE{}{\| g - \EE{}{g} \|^2_2}  = \EE{}{\| g - h\|^2_2} - \| \EE{}{g} - h \|^2_2 \label{eq:variance_decomp}
\end{align}
\textbf{}

For any independent random variables $X_1, X_2, \ldots, X_n \in \R^d$ 
\begin{equation}\label{eq:var_avg}
    \EE{}{\|\frac{1}{n} \sum \limits_{i=1}^n (X_i - \EE{}{X_i}) \|^2} = \frac{1}{n^2} \sum \limits_{i=1}^n \EE{}{\|X_i - \EE{}{X_i}\|^2}
\end{equation}

    \begin{lemma}\cite{agafonov2022flecs}\label{search_direction}
        The search direction $p_k$ in FLECS is equivalent to $p_k = -\A_k \widetilde{g}_k$
    \end{lemma}
    \begin{lemma} \label{lem:1} 
    \cite{agafonov2022flecs} If Assumption \ref{assum:diff} holds, there exist constants $0 < \mu_1 \leq \mu_2$ such that the inverse truncated Hessian approximations $\{ \mathcal{A}_k\}$ generated by FLECS satisfy
    \begin{align}   \label{eq:bnd_Hess1}
        \mu_1 I \preceq \mathcal{A}_k \preceq \mu_2 I,\qquad \text{for all } k=0,1,\dots
    \end{align}
    for some constants $\mu_1,~ \mu_2$.

\end{lemma}

\subsection{Proof of Theorem \ref{thm:str_conv}}
    \begin{customassum}{\ref{assum:diff}}
                The function $F$ is twice continuously differentiable.
    \end{customassum}
    \begin{customassum}{\ref{assum:strong_conv}}
        Each function $f_i(w)$ is $\mu$-strongly convex and $L$-smooth
        \begin{align}
            \mu I \preceq \nabla^2 f_i(w) \preceq L I.
        \end{align}
    \end{customassum}
    \begin{customassum}{\ref{assum:variance}}
        Each $g_k^i$ in Algorithm \ref{alg:main} has bounded variance 
        \begin{equation}
            \EE{}{\|g_k^i - \nabla f_i(w_k)\|} \leq \sigma_i^2, \quad \forall k \geq 0, ~ i = 1, \ldots, n
        \end{equation}
        for constants $\sigma_i < \infty$, $\sigma^2 \eqdef \frac{1}{n} \sum \limits_{i=1}^n \sigma_i^2$.
    \end{customassum}

\begin{lemma}\cite{horvath2019stochastic}
    For all iterations $k \geq 0$ of Algorithm \ref{alg:main} we have:
    \begin{equation}\label{eq:cg_lm1_1}
        \EE{}{\wg_k} = g_k \eqdef \frac{1}{n} \sum \limits_{i=1}^n g_k^i, \quad \EE{Q}{\|\wg_k - \nabla F(w_k)\|^2} \leq \frac{\omega}{n^2}\sum\limits_{i=1}^n \|\nabla f_i(w_k) - h_k^i\|^2, \quad \EE{}{g_k} = \nabla f(w_k);
    \end{equation}
    and
    \begin{equation}\label{eq:cg_lm1_2}
        \EE{}{\|\wg_k - h_*^i\|^2} \leq \frac{2\omega}{n^2}\sum \limits_{i=1}^n \EE{}{\|h_k^i - h_*^i\|^2} + \ls\frac{2\omega}{n} + 1 \rs\frac{1}{n}\sum\limits_{i=1}^n \EE{}{\|\nabla f_i(w_k) - h_*^i\|^2} + (1+\omega)\frac{\sigma^2}{n},
    \end{equation}
    where $h_*^i = \nabla f_i(w^*)$.
\end{lemma}
\begin{proof}
    We prove the first equality in \eqref{eq:cg_lm1_1}:
    \begin{gather*}
        \EE{Q}{\wg_k} = \EE{Q}{\frac{1}{n} \sum \limits_{i=1}^n \wg_k^i} = \EE{Q}{\frac{1}{n} \sum \limits_{i=1}^n (Q(g_k^i - h_k^i) + h_k^i)} \stackrel{\eqref{eq:quantization_def}}{=} \frac{1}{n}\sum \limits_{i=1}^n g_k^i = g_k.
    \end{gather*}
    Now, we prove the second inequality in \eqref{eq:cg_lm1_1}:
    \begin{gather*}
        \EE{Q}{\|\wg_k - g_k\|^2} = \EE{Q}{\| \frac{1}{n} \sum\limits_{i=1}^n (\wg_k^i - \nabla f_i(w_k))\|^2} \\
        = \EE{Q}{\| \frac{1}{n} \sum\limits_{i=1}^n (Q(g_k^i - h_k^i) - (g_k^i - h_k^i))\|^2}\\
        \stackrel{\eqref{eq:quantization_def}, \eqref{eq:var_avg}}{\leq}
        \EE{Q}{\frac{1}{n^2} \sum \limits_{i=1}^n \omega \|g_k^i - h_k^i\|^2} \numberthis{\label{eq:cg_lm1_3}}.
    \end{gather*}
    
    The last equality in \eqref{eq:cg_lm1_1} follows follows from the assumption that each $g_k^i$ is an unbiased estimate of $\nabla f_i(w_k)$. Let $h_* = \nabla F(w^*) = 0$\\
 
    \begin{gather*}
        \EE{}{\|\wg_k - h_*\|^2} \stackrel{\eqref{eq:variance_decomp}}{=} \EE{}{\|\wg_k - g_k\|^2} + \EE{}{\|g_k - h_*\|^2}  \stackrel{\eqref{eq:variance_decomp}}{=}\\
        \EE{}{\|\wg_k - g_k\|^2} + \EE{}{\|g_k - \nabla F(w_k)\|^2} + \EE{}{\|\nabla F(w_k) - h_*\|^2}\numberthis{\label{eq:cg_lm1_4}}
    \end{gather*}
    Then,
    \begin{equation}\label{eq:cg_lm1_5}
        \EE{Q}{\|\wg_k - g_k\|^2} \stackrel{\eqref{eq:cg_lm1_3}}{\leq} \frac{\omega}{n^2}\sum \limits_{i=1}^n \EE{Q}{\|g_k^i - h_k^i\|^2}.
    \end{equation}
    Therefore,
    \begin{gather*}
        \EE{}{\|\wg_k - h_*\|^2} \stackrel{\eqref{eq:cg_lm1_4}, \eqref{eq:cg_lm1_5}}{\leq} \frac{w}{n^2} \sum \limits_{i=1}^n \EE{}{\|g_k^i - h_k^i\|^2} + \EE{}{\|g_k - \nabla F(w_k)\|^2} + \EE{}{\|\nabla F(w_k) - h_* \|^2}\\
        \leq \frac{w}{n^2} \sum \limits_{i=1}^n \EE{}{\| g_k^i - h_k^i\|^2}  + \frac{1}{n} \sum \limits_{i=1}^n \EE{}{\|g_k^i - h_*^i\|^2} + \frac{\sigma^2}{n},
        \numberthis{\label{eq:cg_lm1_6}}
    \end{gather*}
    where the last inequality is valid due to Jensen inequality. Next, 
    \begin{gather*}
        \EE{}{\| g_k^i - h_k^i\|^2} \leq
        \EE{}{\|\nabla f_i(w_k) - h_k^i\|^2} + \EE{}{\|\nabla f_i(w_k) - g_k^i\|^2} \\
        \leq
        \EE{}{\|(\nabla f_i(w_k) -  h_*^i) + ( h_*^i - h_k^i)\|^2} + \sigma_i^2\\
        \stackrel{\|\sum \limits_{i=1}^{t} a_i\|^2 \leq t \sum \limits_{i=1}^t \|a_i\|^2 }{\leq}
        2 \EE{}{\|\nabla f_i(w_k) - h_*^i\|^2} + 2 \EE{}{\|\nabla f_i(w^*) - h_k^i\|^2} + \sigma_i^2. 
        \numberthis{\label{eq:cg_lm1_7}}
    \end{gather*}
    
    Therefore, 
    \begin{gather*}
        \EE{}{\|\wg_k - h_*\|^2} \stackrel{\eqref{eq:cg_lm1_6}, \eqref{eq:cg_lm1_7}}{\leq} 
        \frac{2\omega}{n^2}\sum \limits_{i=1}^n\EE{}{\|\nabla f_i(w_k) -  h_*^i\|^2 + \|h_*^i - h_k^i\|^2} + \frac{\omega}{n^2}\sigma_i^2\\
        + \frac{1}{n}\sum \limits_{i=1}^n \EE{}{\|\nabla f_i(w_k) - h_*^i\|^2} + \frac{\sigma^2}{n}\\
        \leq \frac{2\omega}{n^2}\sum \limits_{i=1}^n\EE{}{\|h_*^i -  h_k^i\|^2} + \left(\frac{2\omega}{n} + 1\right)\frac{1}{n}\sum \limits_{i=1}^n \EE{}{\|\nabla f_i(w_k) - h_*^i\|^2} + (\omega + 1)\frac{\sigma^2}{n}.
    \end{gather*}
\end{proof}


\begin{lemma}\label{lm:decomp}
    \cite{horvath2019stochastic} Let $\gamma_k (\omega + 1) \leq 1$ for any $k \geq 0$. Then for all iterations $k \geq 0$ of Algorithm \ref{alg:main} and all workers $i = 1\ldots n$  we have:
    \begin{equation}\label{eq:lem2}
        \EE{Q}{\|h_{k+1}^i - h_*^i \|^2} \leq (1 - \gamma_k){\|h_k^i - h_*^i\|^2} + \gamma_k{\|\nabla f_i(w_k)-h_*^i\|^2} + \gamma_k \sigma_i^2.
    \end{equation}
\end{lemma}
\begin{proof}
    Since $h_{k+1}^i = h_k^i + \gamma_k Q(g_k^i - h_k^i)$
    \begin{gather*}
        \EE{Q}{\|h_{k+1}^i - h_*^i\|^2} = \EE{Q}{\|\gamma_k Q(g_k^i - h_k^i) + (h_k^i - h_*^i)\|^2} \\
        = \|h_k^i - h_*^i \|^2 + 2 \EE{Q}{\langle\gamma_k Q(g_k^i - h_k^i), h_k^i - h_*^i \rangle} + \EE{Q}{\|\gamma_k Q(g_k^i - h_k^i)\|^2} \\
        \stackrel{\eqref{eq:quantization_def}}{\leq} \|h_k^i - h_*^i \|^2 + 2\gamma_k  \langle g_k^i - h_k^i, h_k^i - h_*^i \rangle + \gamma_k^2(\omega + 1)\|g_k^i - h_k^i\|^2 \\
        \stackrel{\gamma_k(\omega+1) \leq 1}{\leq} \|h_k^i - h_*^i \|^2 + 2\gamma_k \langle g_k^i - h_k^i, h_k^i - h_*^i \rangle + \gamma_k\|g_k^i - h_k^i\|^2 \\
        =\|h_k^i - h_*^i \|^2 + 2\gamma_k  \langle g_k^i - h_k^i, h_k^i - h_*^i \rangle + \gamma_k\langle g_k^i - h_k^i,g_k^i - h_k^i \rangle \\
        = \|h_k^i - h_*^i \|^2 + \gamma_k \langle g_k^i - h_k^i, 2h_k^i - 2h_*^i + g_k^i - h_k^i\rangle\\
        = \|h_k^i - h_*^i \|^2 + \gamma_k \langle g_k^i - h_k^i, h_k^i + g_k^i - 2h_*^i \rangle \\
        = \|h_k^i - h_*^i \|^2 + \gamma_k \|g_k^i - h_*^i\|^2 - \gamma_k \|h_*^i - h_k^i\|^2  \\
        = (1 - \gamma_k)\|h_k^i - h_*^i\|^2 + \gamma_k \|g_k^i - h_*^i\|^2 \\
        \leq (1 - \gamma_k)\|h_k^i - h_*^i\|^2 + \gamma_k \|\nabla f_i(w_k) - h_*^i\|^2 + \gamma_k \sigma_i^2,
    \end{gather*}
    where in the last equality is due to fact that for any vectors $a, b$ we have $\|a - b\|^2 = \langle a - b, a + b\rangle$.
\end{proof}

    
    

\begin{customthm}{\ref{thm:str_conv}}
    Suppose that Assumption \ref{assum:diff}, \ref{assum:strong_conv}, \ref{assum:variance} holds.  Let $Q \in \mathcal{U}(\omega)$.
     Let $\{w_k\}$ be the iterates generated by Algorithm~\ref{alg:main}, where $0 <  \alpha_k = \alpha  \leq \frac{5\mu \mu_1}{2L^2\mu_2^2\left(1+\frac{\omega}{n}\right)}$ and $0 < \gamma_k = \gamma \leq \frac{1}{\omega + 1}$. Define the Lyapunov function
     $$\Psi_{k+1} = (F(w_{k+1}) - F(w_*)) + \frac{cL\mu_2^2\alpha^2}{n}\sum \limits_{i=1}^n \EE{Q}{\|h_{k+1}^i - h_*^i\|^2}$$
     for $0 < c = \min \lb \frac{1- \frac{\alpha\mu\mu_1}{2} - \frac{\omega}{n}}{1 - \gamma}; \frac{\mu}{2\gamma L}\rb$. Then for all $k \geq 0$:
     \begin{equation}
         \EE{Q}{\Psi_{k}} \leq \left(1 - \frac{\alpha \mu\mu_1}{2}\right)^{k+1}\Psi_0 + \left(\frac{\omega + 1}{2n} + \gamma c\right)\frac{2L \mu_2^2 \alpha}{\mu \mu_1}  \sigma^2.
     \end{equation}
\end{customthm}

\begin{proof}
    \begin{gather*}
        \EE{}{F(w_{k+1})} \leq \EE{}{F(w_k) - \nabla F(w_k)^T(-\alpha A_k \wg_k) + \frac{L}{2}\|\alpha A_k \wg_k\|^2} \\
        = \EE{}{F(w_k)} - \alpha \EE{}{\nabla F(w_k)^TA_k \nabla F(w_k)} + \frac{L\mu_2^2\alpha^2}{2}\EE{}{\|\wg_k\|^2}  \\
        \stackrel{\text{Lem. \ref{lem:1}}}{\leq} \EE{}{F(w_k)} - \alpha \mu_1 \EE{}{\|\nabla F(w_k)\|^2} + \frac{L\mu_2^2\alpha^2}{2}\EE{}{\|\wg_k - h_*\|^2} \\
        \stackrel{\eqref{eq:cg_lm1_2}}{\leq} \EE{}{F(w_k)} - \alpha \mu_1 \EE{}{\|\nabla F(w_k)\|^2}\\
        + \frac{L\mu_2^2\alpha^2}{2} \left( \frac{2\omega}{n^2}\sum \limits_{i=1}^n \EE{}{\|h_k^i - h_*^i\|^2} + \left(\frac{2\omega}{n} + 1 \right)\sum \limits_{i=1}^n \EE{}{\|\nabla f_i(w_k) - h_*^i\|^2} + (\omega + 1)\frac{\sigma^2}{n}\right) \\
        \leq \EE{}{F(w_k)} - \alpha \mu_1 \EE{}{\|\nabla F(w_k)\|^2} \\
        +
        \frac{L\alpha^2\mu_2^2}{2}\EE{}{\|\nabla F(w_k)\|^2} - \frac{L\alpha^2\mu_2^2}{2}\EE{}{\|\nabla F(w_k)\|^2} 
        \\+ \frac{L\mu_2^2\alpha^2}{2} \left( \frac{2\omega}{n^2}\sum \limits_{i=1}^n \EE{}{\|h_k^i - h_*^i\|^2} + \left(\frac{2\omega}{n} + 1 \right)\sum \limits_{i=1}^n \EE{}{\|\nabla f_i(w_k) - h_*^i\|^2} + (\omega + 1)\frac{\sigma^2}{n}\right). 
    \end{gather*}
    By strong convexity of $F$ we have $2\mu (F(w_k) - F(w_*)) \leq \|F(w_k)\|^2$. By $L$-Lipschitz continuity of each $f_i$ we have
    $f_i(w_*) + \langle\nabla f_i(w_*), w_k - w_*\rangle + \frac{1}{2L}\|\nabla f_i(w_k) - \nabla f_i(w_*)\|^2 \leq f_i(w_k)$. Therefore,
    \begin{gather*}
        \EE{}{F(w_k) - F(w_*)} = \frac{1}{n} \sum \limits_{i=1}^n \EE{}{(f_i(w_k) - f_i(w_*))} \\
        \geq \frac{1}{n} \sum \limits_{i=1}^n \EE{}{\left( \langle\nabla f_i(w_*), w_k - w_*\rangle + \frac{1}{2L}\|\nabla f_i(w_k) - \nabla f_i(w_*)\|^2 \right)} \\
        =  \EE{}{\left( \langle \frac{1}{n} \sum \limits_{i=1}^n \nabla f_i(w_*), w_k - w_*\rangle + \frac{1}{2L}  \frac{1}{n} \sum \limits_{i=1}^n  \|\nabla f_i(w_k) - \nabla f_i(w_*)\|^2 \right)} \\
        = \frac{1}{2L}\frac{1}{n}\sum \limits_{i=1}^n  \EE{}{\|\nabla f_i(w_k) - \nabla f_i(w_*)\|^2}.
    \end{gather*}
    
    Then,
    \begin{gather*}
        \EE{}{F(w_{k+1})} \leq  \EE{}{F(w_k)} - 2 \alpha \mu \mu_1 \EE{}{(F(w_k) - F(w_*))}  +
        {L^2\alpha^2\mu_2^2}\EE{}{(F(w_k) - F(w_*))}  
        \\- \frac{L\alpha^2\mu_2^2}{2}\EE{}{\|\nabla F(w_k)\|^2}+ \frac{L\mu_2^2\alpha^2}{2} \left( \frac{2\omega}{n^2}\sum \limits_{i=1}^n \EE{}{\|h_k^i - h_*^i\|^2} + \left(\frac{2\omega}{n} + 1 \right)2L\EE{}{(F(w_k) - F(w_*))} + (\omega +1)\frac{\sigma^2}{n} \right) \\
        = \EE{}{F(w_k)} - (2\alpha\mu\mu_1 - 2L^2\mu_2^2\alpha^2(\omega + 1))\EE{}{(F(w_k) - F(w_*))} - \frac{L\alpha^2\mu_2^2}{4}\EE{}{\|\nabla F(w_k)\|^2}\\
        + \frac{L\mu_2^2\alpha^2\omega}{n^2}\sum \limits_{i=1}^n \EE{}{\|h_k^i - h_*^i\|^2} + (\omega +1)\frac{L\mu_2^2\alpha^2\sigma^2}{2n}\\
        \leq \EE{}{F(w_k)} - \frac{\alpha \mu \mu_1}{2}\EE{}{(F(w_k) - F(w_*))} - \frac{L\alpha^2\mu_2^2}{4}\EE{}{\|\nabla F(w_k)\|^2} \\
        + \frac{L\mu_2^2\alpha^2\omega}{n^2}\sum \limits_{i=1}^n \EE{}{\|h_k^i - h_*^i\|^2} +  (\omega +1)\frac{L\mu_2^2\alpha^2\sigma^2}{2n}
    \end{gather*}
    
    Where the last inequality holds due to the choice of learning rate $0 < \alpha \leq \frac{5\mu \mu_1}{2L^2\mu_2^2\left(1+\frac{\omega}{n}\right)}$.
    
    By subtracting $F(w_*)$ from the LHS and the RHS, we have
    
    \begin{gather}
        \EE{}{F(w_{k+1}) - F(w_*)} 
        \leq \left(1 - \frac{\alpha \mu\mu_1}{2}\right)\EE{}{(F(w_k) - F(w_*))} - \frac{L\alpha^2\mu_2^2}{4}\EE{}{\|\nabla F(w_k)\|^2} \\
        + \frac{L\mu_2^2\alpha^2\omega}{n^2}\sum \limits_{i=1}^n \EE{}{\|h_k^i - h_*^i\|^2} +  (\omega +1)\frac{L\mu_2^2\alpha^2\sigma^2}{2n}.\numberthis{\label{eq:f_diff_bound}}
    \end{gather}
    
    Let us define Lyapunov function $\Psi_{k+1}$ as 
    \begin{equation}
        \Psi_{k+1} = \EE{}{(F(w_{k+1}) - F(w_*))} + \frac{cL\mu_2^2\alpha^2}{n}\sum \limits_{i=1}^n \EE{}{\|h_{k+1}^i - h_*^i\|^2}.
    \end{equation}
    
    Then 
    \begin{gather*}
        \Psi_{k+1} \stackrel{\eqref{eq:f_diff_bound}}{\leq}
        \left(1 - \frac{\alpha \mu\mu_1}{2}\right)\EE{}{(F(w_k) - F(w_*))} - \frac{L\alpha^2\mu_2^2}{4}\EE{}{\|\nabla F(w_k)\|^2} \\
        + \frac{L\mu_2^2\alpha^2\omega}{n^2}\sum \limits_{i=1}^n \EE{}{\|h_k^i - h_*^i\|^2} + \frac{cL\mu_2^2\alpha^2}{n}\sum \limits_{i=1}^n \EE{}{\|h_{k+1}^i - h_*^i\|^2} +  (\omega +1)\frac{L\mu_2^2\alpha^2\sigma^2}{2n} \\
        \stackrel{\eqref{eq:lem2}}{\leq}
        \left(1 - \frac{\alpha \mu\mu_1}{2}\right)\EE{}{(F(w_k) - F(w_*))} - \frac{L\alpha^2\mu_2^2}{4}\EE{}{\|\nabla F(w_k)\|^2} + \frac{L\mu_2^2\alpha^2\omega}{n^2}\sum \limits_{i=1}^n \EE{}{\|h_k^i - h_*^i\|^2}\\
         + 
        \frac{cL\mu_2^2\alpha^2}{n}\sum \limits_{i=1}^n\left( (1 - \gamma)\EE{}{\|h_k^i - h_*^i\|^2} +  (\omega +1)\frac{L\mu_2^2\alpha^2\sigma^2}{2n} + \gamma \EE{}{\|\nabla f_i(w_k) - h_*^i\|^2} + \gamma \sigma_i^2\right ) \numberthis{\label{eq:lyap_bnd}}
    \end{gather*}
    
    Then by $L$-Lipschitz  continuity of each $f_i$ and $\mu$ strong convexity of $F$ we have
    \begin{gather*}
        - \frac{L\alpha^2\mu_2^2}{4}\EE{}{\|\nabla F(w_k)\|^2}+ \frac{cL\mu_2^2\alpha^2\gamma}{n}\sum \limits_{i=1}^n \EE{}{\|\nabla f_i(w_k) - h_*^i\|^2} \\
        \leq - \frac{\mu L\alpha^2\mu_2^2}{2}\EE{}{(F(w_k) - F(w_*))} + 2c\gamma L^2\alpha^2\mu_2^2 \EE{}{(F(w_k) - F(w_*))} \\
        \leq \left(2c\gamma L^2\alpha^2\mu_2^2 - \frac{\mu L\alpha^2\mu_2^2}{2}\right)\EE{}{(F(w_k) - F(w_*))} \leq 0,
        \numberthis{\label{eq:grad_bnd}}
    \end{gather*}
    where the last inequality is due to the choice of $c$ and $\gamma$ as 
    $\gamma \leq \frac{\mu}{2cL}$.
    
    
    By assumption on $c$, \eqref{eq:lyap_bnd} and \eqref{eq:grad_bnd} we have
    \begin{gather*}
        \Psi_{k+1} \leq \left(1 - \frac{\alpha \mu\mu_1}{2}\right)\EE{}{(F(w_k) - F(w_*))} + (1 - \gamma)\frac{cL\mu_2^2\alpha^2}{n}\sum \limits_{i=1}^n \EE{}{\|h_k^i - h_*^i\|^2} \\
        + \frac{L\mu_2^2\alpha^2\omega}{n^2}\sum \limits_{i=1}^n \EE{}{\|h_k^i - h_*^i\|^2} +  \left(\frac{\omega + 1}{2n} + \gamma c\right) L \mu_2^2 \alpha^2 \sigma^2\\
        \leq \left(1 - \frac{\alpha \mu\mu_1}{2}\right)\EE{}{(F(w_k) - F(w_*))} + \left(1 - \frac{\alpha \mu\mu_1}{2}\right)\frac{cL\mu_2^2\alpha^2}{n}\sum \limits_{i=1}^n \EE{}{\|h_k^i - h_*^i\|^2} +  \left(\frac{\omega + 1}{2n} + \gamma c\right) L \mu_2^2 \alpha^2 \sigma^2.
    \end{gather*}
    
    Finally, 
    \begin{gather*}
        \Psi_{k+1} \leq \left(1 - \frac{\alpha \mu\mu_1}{2}\right)\Psi_k + \left(\frac{\omega + 1}{2n} + \gamma c\right) L \mu_2^2 \alpha^2 \sigma^2 \\
        \leq \left(1 - \frac{\alpha \mu\mu_1}{2}\right)^{k+1}\Psi_0 + \left(\frac{\omega + 1}{2n} + \gamma c\right)L \mu_2^2 \alpha^2 \sigma^2 \sum \limits_{t=0}{k} \left(1 - \frac{\alpha \mu\mu_1}{2}\right)^t \\
        \leq 
        \left(1 - \frac{\alpha \mu\mu_1}{2}\right)^{k+1}\Psi_0 + \left(\frac{\omega + 1}{2n} + \gamma c\right)\frac{2L \mu_2^2 \alpha}{\mu \mu_1}  \sigma^2,
        \end{gather*}
        where the last inequality is due to estimate 
        $\sum \limits_{t=0}^{k} (1 - \frac{\alpha \mu \mu_1}{2})^t \leq \frac{2}{\alpha\mu \mu_1}$.
\end{proof}

    
    

\subsection{Proof of Theorem \ref{thm:nonconvex}}


\begin{lemma}\label{feedback_bound_nonconvex}\cite{Mishchenko2019}
    Let $x^* \in X^*$, such that $X^*$ is the set of solutions for \eqref{eq:problem}, and define $h^i_* = \nabla f_i(x^*)$, we have for each worker $i \in [n]$, the first and second moments of $h^i_{k+1}$ are equal to:
    \begin{align}
        \EE{\aa{Q}}{h_{k+1}^i} &= (1-\gamma_k) h_k^i + \gamma_k g_k^i \label{eq:error-feedback1}\\
        \EE{\aa{Q}}{\| h_{k+1}^i - h_*^i\|^2_2} &\leq \ls 1 - \gamma_k \rs \| h_k^i - h_*^i\|^2_2 + \gamma_k \|g_k^i - h_*^i \|^2_2+ \ls \gamma_k^2  \omega- \gamma_k (1-\gamma_k)\rs \|g_k^i - h_k^i \|_2^2 \label{eq:error-feedback_second_moment}
    \end{align}
\end{lemma}
\begin{proof}
    Since $h_{k+1}^i = h_k^i + \gamma_k c^i_k$\\
    \begin{align*}
        \EE{\aa{Q}}{h_{k+1}^i} &= h_k^i + \gamma_k \EE{\aa{Q}}{c_i^k}\\
        &= h_k^i + \gamma_k (g^i_k - h_k^i).
    \end{align*}
Secondly:
\begin{align*}
    \EE{\aa{Q}}{\| h_{k+1}^i - h_*^i\|^2_2} &\stackrel{\eqref{eq:variance_decomp}}{=} \left\|\EE{\aa{Q}}{ h_{k+1}^i} - h_*^i\right\|^2_2 + \EE{\aa{Q}}{\left\| h_{k+1}^i - \EE{\aa{Q}}{ h_{k+1}^i} \right\|^2_2}\\
    &\stackrel{\eqref{eq:error-feedback1}}{=}  \left\|(1-\gamma_k) h_k^i + \gamma_k  g_k^i - h_*^i\right\|^2_2 + \gamma_k^2 \EE{\aa{Q}}{\left\| c_i^k - \EE{\aa{Q}}{ c_i^k} \right\|^2_2}\\
    &= \left\|(1-\gamma_k) \lp h_k^i - h_*^i \rp + \gamma_k \lp g_k^i - h_*^i \rp\right\|^2_2 + \gamma_k^2 \EE{\aa{Q}}{\left\| c_i^k - \EE{\aa{Q}}{ c_i^k} \right\|^2_2}\\
    &\stackrel{\eqref{eq:quantization_def}}{\leq} \left\|(1-\gamma_k) \lp h_k^i - h_*^i \rp + \gamma_k \lp g_k^i - h_*^i \rp\right\|^2_2 + \gamma_k^2 \omega \| g_k^i - h_k^i \|_2^2\\
    &\stackrel{\eqref{eq:norm_convex_comb}}{=} \ls 1 - \gamma_k \rs \| h_k^i - h_*^i\|^2_2 + \gamma_k \|g_k^i - h_*^i \|^2_2 -\gamma_k (1-\gamma_k) \|g_k^i - h_k^i \|_2^2 \\
    & + \gamma_k^2 \omega \| g_k^i - h_k^i \|_2^2\\
    &= \ls 1 - \gamma_k \rs \| h_k^i - h_*^i\|^2_2 + \gamma_k \|g_k^i - h_*^i \|^2_2+ \ls \gamma_k^2  \omega- \gamma_k (1-\gamma_k)\rs \|g_k^i - h_k^i \|_2^2. 
\end{align*}
\end{proof}

    \begin{customassum}{\ref{assum:Lsmooth}}
        The function $F$ is $L$-smooth.
    \end{customassum}
    
    \begin{customassum}{\ref{asymmetry}}
    (Bounded data dissimilarity). There exists constant $ \zeta \geq 0$ such that $\forall x \in \R^d$ 
    \begin{align}
        \frac{1}{n} \sum_{i=1}^n \| \nabla f_i(x)  - \nabla F(x)\|^2_2 \leq \zeta^2 
    \end{align}
    In particular, $\zeta = 0$, implies that all datasets stored in the $n$ devices are drawn from the same data distribution $\mathcal{D}$.
    \end{customassum}

\begin{customthm}{\ref{thm:nonconvex}}
    Let $S = \{ w_0, w_1, \dots, w_{k-1} \}$ be generated using Algorithm \ref{alg:main}, and $\bar{w}$ be sampled uniformily at random from $S$, for $\alpha \leq \ \sqrt{\frac{n}{2Lw(w+1)\mu_2^2}}$ and $\gamma_k \leq  \frac{1 + \sqrt{1 - \frac{2 L \alpha^2 w(w+1)\mu_2^2}{n}} }{2(w+1)}$, and a parameter $c$ such as $c < \frac{\mu_1}{L \alpha \gamma_k } - \frac{\mu^2_2}{2\gamma_k} $  we have:
    \begin{align}
        \EE{Q}{\|\nabla F(\bar{w})\|_2^2}  &\leq 2 \frac{\bk^0}{ k\alpha \ls 2\mu_1 -  L \alpha \mu_2^2   - 2c L \alpha \gamma _k\rs} +  \frac{4 c L \alpha  }{  2\mu_1 -  L \alpha \mu_2^2   - 2c L \alpha \gamma _k}  \zeta^2 \nonumber\\&+  \frac{\mu_2^2+ 2 c }{  2\mu_1 - \ls L \alpha \mu_2^2  \rs - 2c L \alpha } L  \sigma^2 \nonumber\\
    \end{align}
with $\bk^k = F(w_k) - F^* + c \frac{L \alpha^2}{2}  \frac{1}{n} \sum_{i=1}^{n} \|h_k^i - h_*^i \|^2_2$
\end{customthm}

\begin{proof}
    We have $w_{k+1} = w_k - \alpha_k A_k \wg_k$, therefore:
    \begin{align*}
        \EE{}{F(w_{k+1})} &= \EE{}{F(w_k - \alpha_k A_k \wg_k)}\\
        &\stackrel{\eqref{assum:strong_conv}}{\leq} \EE{}{F(w_{k})} - \alpha \EE{}{\la \nabla F(w_k), A_k \wg_k \ra} + \frac{L \alpha^2 }{2} \EE{}{\|A_k \wg_k\|^2}\\
        &\stackrel{\eqref{eq:bnd_Hess1},\eqref{eq:cg_lm1_1}}{\leq} \EE{}{F(w_{k})} - \alpha \mu_1 \EE{}{\|\nabla F(w_k)  \|^2} +   \frac{L \alpha^2 \mu_2^2 }{2} \EE{}{\|\wg_k\|^2}\\
        &= \EE{}{F(w_{k})} - \alpha \mu_1 \EE{}{\|\nabla F(w_k)  \|^2} +   \frac{L \alpha^2 \mu_2^2 }{2} \lp \EE{}{\|\wg_k - g_k \|^2} + \EE{}{\|g_k\|^2}\rp\\
        &\stackrel{\eqref{eq:cg_lm1_1}}{\leq} \EE{}{F(w_{k})} - \alpha \mu_1 \EE{}{\|\nabla F(w_k)  \|^2} +   \frac{L \alpha^2 \mu_2^2 }{2} \lp \frac{\omega}{n^2}\sum\limits_{i=1}^n \|g_k^i - h_k^i\|^2 + \EE{}{\|\nabla F(w_k)\|^2} + \sigma^2 \rp\\ 
        &= \EE{}{F(w_{k})} + \alpha\ls \frac{L \alpha \mu_2^2 }{2} -  \mu_1  \rs \EE{}{\|\nabla F(w_k)\|^2} + \frac{L \alpha^2 \mu_2^2 }{2} \frac{\omega}{n^2}\sum\limits_{i=1}^n \|g_k^i - h_k^i\|^2 + \frac{L \aa{\alpha^2} \mu_2^2 }{2} \sigma^2
    \end{align*}\\
Define $\bk^k = F(w_k) - F^* + c \frac{L \alpha^2}{2}  \frac{1}{n} \sum_{i=1}^{n} \|h_k^i - h_*^i  \|^2_2$, where $h_*^i = \nabla f_i(w^*) $.
\begin{align*}
    \frac{1}{n} \sum_{i=1}^{n} \EE{}{\| h_{k+1}^i - h_*^i  \|^2_2 | w_k} &\stackrel{\eqref{eq:error-feedback_second_moment}}{\leq} \frac{1}{n}  \sum_{i=1}^{n} \bigg[  \ls 1 - \gamma_k \rs \| h_k^i - h_*^i\|^2_2 + \gamma_k \|g_k^i - h_*^i \|^2_2+ \ls \gamma_k^2  \omega- \gamma_k (1-\gamma_k)\rs \|g_k^i - h_k^i \|_2^2 \bigg] \\
    &= \frac{1-\gamma_k}{n} \sum_{i=1}^{n}  \| h_k^i - h_*^i \|^2_2 + \frac{\gamma_k}{n} \sum_{i=1}^{n} \|g_k^i - h_*^i \|^2 + \frac{\ls \gamma_k^2  \omega- \gamma_k (1-\gamma_k)\rs}{n} \sum_{i=1}^{n} \|g_k^i - h_k^i \|_2^2 \\
    &\stackrel{\|a+b\|^2 \leq 2\|a\|^2 + 2\|b\|^2}{\leq}  \frac{1-\gamma_k}{n} \sum_{i=1}^{n}  \| h_k^i - h_*^i\|^2_2 + \frac{2\gamma_k}{n}  \sum_{i=1}^n\| g_k^i \|^2_2 + \\& \frac{2\gamma_k}{n}  \sum_{i=1}^n\| h_*^i   - \underbrace{\nabla F(w^*)}_{=0} \|^2_2 + \frac{\ls \gamma_k^2  \omega- \gamma_k (1-\gamma_k)\rs}{n} \sum_{i=1}^{n} \|g_k^i - h_k^i \|_2^2 \\
    &\leq  \frac{1-\gamma_k}{n} \sum_{i=1}^{n}  \| h_k^i - h_*^i\|^2_2 + \frac{2\gamma_k}{n}  \sum_{i=1}^n\| \nabla f_i(w_k) \|^2_2 + \frac{2\gamma_k}{n}  \sum_{i=1}^n \sigma_i^2  \\& +   \frac{2\gamma_k}{n}  \sum_{i=1}^n\| h_*^i  - \underbrace{\nabla F(w^*)}_{=0} \|^2_2 + \frac{\ls \gamma_k^2  \omega- \gamma_k (1-\gamma_k)\rs}{n} \sum_{i=1}^{n} \|g_k^i - h_k^i \|_2^2 \\
    &\stackrel{\eqref{asymmetry} + \eqref{eq:variance_decomp}}{\leq}  \frac{1-\gamma_k}{n} \sum_{i=1}^{n}  \| h_k^i - h_*^i\|^2_2 + 2 \gamma_k \|\nabla F(w_k) \|_2^2 + 4 \gamma_k \zeta^2 \\&+  \frac{\ls \gamma_k^2  \omega- \gamma_k (1-\gamma_k)\rs}{n} \sum_{i=1}^{n} \|g_k^i - h_k^i \|_2^2 + 2 \gamma_k \sigma^2 \\
\end{align*}
Therefore:
\begin{align*}
    \EE{}{\bk^{k+1}} &= \EE{}{F(w_{k+1})} - F^* + c \frac{L \alpha^2}{2} \frac{1}{n} \sum_{i=1}^{n} \EE{}{\|h_{k+1}^i - h_*^i   \|^2_2}\\
    &\leq  \EE{}{F(w_{k})} + \alpha\ls \frac{L \alpha \mu_2^2 }{2} -  \mu_1  \rs \EE{}{\|\nabla F(w_k)\|^2} + \frac{L \alpha^2 \mu_2^2 }{2} \frac{\omega}{n^2}\sum\limits_{i=1}^n \|g_k^i - h_k^i\|^2 + \frac{L \br{\alpha^2} \mu_2^2 }{2} \sigma^2  - F^*
    \\&  + c \frac{L \alpha^2}{2} \bigg[\frac{1-\gamma_k}{n} \sum_{i=1}^{n}  \| h_k^i - h_*^i\|^2_2 + 2 \gamma_k \|\nabla F(w_k) \|_2^2
    \\&+ 4 \gamma_k \zeta^2 +  \frac{\ls \gamma_k^2  \omega- \gamma_k (1-\gamma_k)\rs}{n} \sum_{i=1}^{n} \|g_k^i - h_k^i \|_2^2 + 2 \gamma_k \sigma^2 \bigg ]\\
    &= \EE{}{F(w_{k})} - F^* + c \frac{L \alpha^2 (1 - \gamma_k)}{2n} \sum_{i=1}^{n} \| h_k^i - h^*_i\|_2^2  - \alpha \ls \mu_1 - \frac{L \alpha \mu_2^2}{2} - c L \alpha \gamma_k  \rs \| \nabla F(w_k)\|^2_2 \\&+ 2 c L \alpha^2 \gamma_k \zeta^ 2 + \ls \frac{\mu_2^2}{2} +c\gamma_k \rs L \br{\alpha^2} \sigma^2 +\ls \underbrace{\frac{L \alpha^2 \mu_2^2 }{2}  \frac{\omega}{n^2} + \frac{\ls \gamma_k^2  \omega- \gamma_k (1-\gamma_k)\rs}{n}}_{:= T(\gamma_k, \alpha)} \rs \sum_{i=1}^{n} \|g_k^i - h_k^i \|_2^2
\end{align*}

A key moment in the proof is to notice that $T(\gamma_k, \alpha) \leq 0$ for our choice of $\gamma_k$ and $\alpha$.

In fact, we have: 
\begin{align*}
    T(\gamma_k, \alpha) \leq 0 &\Leftrightarrow \frac{1}{n} \ls \frac{L \alpha^2 \mu_2^2 }{2}  \frac{\omega}{n} + \ls \gamma_k^2  \omega- \gamma_k (1-\gamma_k)\rs \rs \leq 0\\
    &\Leftarrow  \begin{cases} 
      \alpha \leq \sqrt{\frac{n}{2Lw(w+1)\mu_2^2}} \\
      \gamma_k \leq \frac{\sqrt{1 - \frac{2 L \alpha^2 w(w+1)\mu_2^2}{n}} + 1}{2(w+1)}
   \end{cases}
\end{align*}

Therefore, we have: 
\begin{align*}
    \EE{}{\bk^{k+1}} &\leq \EE{}{F(w_{k})} - f^* +  c \frac{L \alpha^2 }{2n} \sum_{i=1}^{n} \| h_k^i - h^*_i\|_2^2  - \alpha \ls \underbrace{\mu_1 - \frac{L \alpha \mu_2^2}{2} - c L \alpha \gamma_k}_{ > 0 \text{ by our condition on } c} \rs \| \nabla F(w_k)\|^2_2  \\&+2 c L \alpha^2 \gamma_k \zeta^ 2 + \ls \frac{\mu_2^2}{2} +c\gamma_k \rs L \br{\alpha^2} \sigma^2\\
    \EE{}{\bk^{k+1}} &\leq \EE{}{\bk^{k}} - \alpha \ls \mu_1 - \frac{L \alpha \mu_2^2}{2} - c L \alpha  \gamma_k\rs \| \nabla F(w_k)\|^2_2  \\&+2 c L \alpha^2 \gamma_k \zeta^ 2 + \ls \frac{\mu_2^2}{2} +c\gamma_k \rs L \br{\alpha^2} \sigma^2\\
\end{align*}

Therefore:
\begin{align*}
    \EE{}{\|\nabla F(w_k)\|_2^2} &\leq 2 \frac{\EE{Q}{\bk^k} - \EE{Q}{\bk^{k+1}}}{ \alpha \ls 2\mu_1 - \ls L \alpha \mu_2^2  \rs - 2c L \alpha \gamma_k\rs} +  \frac{4 c L \alpha^2 \gamma_k }{\alpha \ls 2\mu_1 - \ls L \alpha \mu_2^2  \rs - 2c L \alpha \gamma_k\rs}  \zeta^2 \\&+  \frac{\mu_2^2+ 2 c\gamma_k }{\alpha \ls 2\mu_1 - \ls L \alpha \mu_2^2  \rs - 2c L \alpha \gamma_k \rs} L \br{\alpha^2} \sigma^2\\
\end{align*}

Summing from $0$ and $k-1$  , simplifying the telescopic terms yiels:

\begin{align*}
   \sum_{j=0}^{k-1} \EE{}{\|\nabla F(w_j)\|_2^2} &\leq 2 \frac{\bk^0 - \EE{Q}{\bk^{k}}}{ \alpha \ls 2\mu_1 -  L \alpha \mu_2^2   - 2c L \alpha \gamma _k\rs} + k \frac{4 c L \alpha^2  }{\alpha \ls 2\mu_1 -  L \alpha \mu_2^2   - 2c L \alpha \gamma _k\rs}  \zeta^2 \\&+  k\frac{\mu_2^2+ 2 c \gamma_k }{\alpha \ls 2\mu_1 -  L \alpha \mu_2^2   - 2c L \alpha \gamma _k\rs} L \br{\alpha^2} \sigma^2\\
\end{align*}
Finally:
\begin{align*}
    \frac{1}{k}\sum_{j=0}^{k-1} \EE{}{\|\nabla F(w_j)\|_2^2} &\leq 2 \frac{\bk^0 - \EE{Q}{\bk^{k}}}{ k\alpha \ls 2\mu_1 -  L \alpha \mu_2^2   - 2c L \alpha \gamma _k\rs} +  \frac{4 c L \alpha  }{ 2\mu_1 -  L \alpha \mu_2^2   - 2c L \alpha \gamma _k}  \zeta^2 \\&+  \frac{\mu_2^2+ 2 c \gamma_k }{ 2\mu_1 -  L \alpha \mu_2^2   - 2c L \alpha \gamma _k} L \br{\alpha} \sigma^2\\
\end{align*}
We can drop $\EE{Q}{\bk^k}$ because it is positive and that concludes the proof.
\end{proof}

\begin{customcol}{\ref{cor:nonconvex}}
 Set $\gamma_k = \gamma$  , $\alpha = \frac{2 \mu_1 - 1}{L(\mu_2^2 + 2c \gamma) \sqrt{K}}$ and $h_0 = 0$, after $K$ iterations of algorithm \ref{alg:main}, in the nonconvex setting, the error $\epsilon$ is at worst $\frac{2}{\sqrt{K}} \frac{ L(\mu_2^2 + 2c \gamma)}{  \ls 2 \mu_1 - 1 \rs } \bk^0 + \frac{1}{\sqrt{K}} \frac{4 c \ls 2 \mu_1 - 1 \rs   }{ \mu_2^2 + 2c \gamma}  \zeta^2 + \frac{1}{\sqrt{K}} \frac{\ls\mu_2^2+ 2 c \gamma\rs \ls 2\mu_1 - 1\rs }{ \mu_2^2 + 2c \gamma }  \sigma^2$.
\end{customcol}
\begin{proof}
    It's easy to see that by our choice of $\gamma_k,  \alpha$ and $h_0$\\
    we have $2\mu_1 -  L \alpha \mu_2^2   - 2c L \alpha \gamma _k \geq 1$\\
    Therefore, after the $K$ steps, the error $\epsilon$ is upper bounded by:
    
     \begin{align*}
         &2 \frac{\bk^0}{ k\alpha \ls 2\mu_1 -  L \alpha \mu_2^2   - 2c L \alpha \gamma _k\rs} +  \frac{4 c L \alpha  }{ 2\mu_1 -  L \alpha \mu_2^2   - 2c L \alpha \gamma _k}  \zeta^2 +  \frac{\mu_2^2+ 2 c \gamma_k }{ 2\mu_1 -  L \alpha \mu_2^2   - 2c L \alpha \gamma _k} L \br{\alpha} \sigma^2 
         \\ &\leq \frac{2}{\sqrt{K}} \frac{ L(\mu_2^2 + 2c \gamma)}{  \ls 2 \mu_1 - 1 \rs } \bk^0 + \frac{1}{\sqrt{K}} \frac{4 c \ls 2 \mu_1 - 1 \rs   }{ \mu_2^2 + 2c \gamma}  \zeta^2 + \frac{1}{\sqrt{K}} \frac{\ls\mu_2^2+ 2 c \gamma\rs \ls 2\mu_1 - 1\rs }{ \mu_2^2 + 2c \gamma }  \sigma^2
     \end{align*}
    
\end{proof}

\end{document}